\documentclass[twoside,11pt]{article}

\usepackage{caption, subcaption}

\usepackage[preprint]{jmlr2e} 

\hypersetup{hidelinks} 

\newcommand{\ie}{i.e.}
\newcommand{\eg}{e.g.}

\graphicspath{ {./figures/} }

\usepackage{amsmath, mathtools, bm, dsfont}

\newcommand{\N}{\mathbb{N}} 
\newcommand{\R}{\mathbb{R}} 
\newcommand{\B}{\mathbb{B}} 

\newcommand{\range}[1]{[#1]} 
\renewcommand{\O}{\mathcal{O}} 

\renewcommand{\vec}[1]{\bm{#1}} 
\newcommand{\mat}[1]{\bm{#1}} 
\newcommand{\set}[1]{\mathcal{#1}} 
\newcommand{\norm}[1]{\left\lVert#1\right\rVert} 

\newcommand{\1}{\mathds{1}} 
\newcommand{\gw}{w} 
\newcommand{\dist}{\mathrm{d}} 
\newcommand{\sign}{\mathrm{sign}} 
\newcommand{\Thres}{\mathrm{Thres}} 

\newcommand{\E}{\mathbb{E}} 
\newcommand{\bP}{\mathbb{P}} 
\newcommand{\Unif}{\mathrm{Unif}} 
\newcommand{\Normal}{N} 

\newcommand{\code}{\texttt} 

\usepackage{algorithm, algpseudocode, dashrule, calc}
\algrenewcommand\algorithmicrequire{\textbf{Input:}}
\algrenewcommand\algorithmicensure{\textbf{Output:}}
\algnewcommand\And{\textbf{and}}
\algnewcommand\Or{\textbf{or}}

\makeatletter
\newcounter{phase}[algorithm]
\newlength{\phaserulewidth}
\newcommand{\setphaserulewidth}{\setlength{\phaserulewidth}}
\newcommand{\phase}[1]{%
    \vspace{-1.25ex}\Statex\leavevmode\llap{\hdashrule{\dimexpr\labelwidth+\labelsep}{\phaserulewidth}{4\phaserulewidth}}\hdashrule{\linewidth}{\phaserulewidth}{4\phaserulewidth}
    \Statex\strut\refstepcounter{phase}
    \vspace{-2.75ex}
    }
\makeatother

\setphaserulewidth{.7pt}

\usepackage{lastpage}
\jmlrheading{25}{2024}{1-\pageref{LastPage}}{10/23; Revised 7/24}{10/24}{23-1376}{Sjoerd Dirksen, Patrick Finke and Martin Genzel}

\ShortHeadings{Memorization With Neural Nets: Going Beyond the Worst Case}{Dirksen, Finke and Genzel}
\firstpageno{1}

\begin{document}

\title{Memorization With Neural Nets:\\Going Beyond the Worst Case}

\author{\name Sjoerd Dirksen \email s.dirksen@uu.nl \\
       \addr Mathematical Institute\\
       Utrecht University\\
       3584 CD Utrecht, Netherlands
       \AND
       \name Patrick Finke \email p.g.finke@uu.nl \\
       \addr Mathematical Institute\\
       Utrecht University\\
       3584 CD Utrecht, Netherlands
       \AND
       \name Martin Genzel\thanks{Work done while at Utrecht University.} \email martin.genzel@merantix-momentum.com \\
       \addr Merantix Momentum GmbH\\
       13355 Berlin, Germany}
       
\editor{Mahdi Soltanolkotabi}

\maketitle

\begin{abstract}
    In practice, deep neural networks are often able to easily interpolate their training data. To understand this phenomenon, many works have aimed to quantify the memorization capacity of a neural network architecture: the largest number of points such that the architecture can interpolate any placement of these points with any assignment of labels. For real-world data, however, one intuitively expects the presence of a benign structure so that interpolation already occurs at a smaller network size than suggested by memorization capacity.
    In this paper, we investigate interpolation by adopting an instance-specific viewpoint. We introduce a simple randomized algorithm that, given a fixed finite data set with two classes, with high probability constructs an interpolating three-layer neural network in polynomial time. The required number of parameters is linked to geometric properties of the two classes and their mutual arrangement. As a result, we obtain guarantees that are independent of the number of samples and hence move beyond worst-case memorization capacity bounds. We verify our theoretical result with numerical experiments and additionally investigate the effectiveness of the algorithm on MNIST and CIFAR-10.
\end{abstract}


\section{Introduction}
\label{sec:introduction}

The \emph{bias-variance tradeoff}~\citep{shalev2014understanding, hastie2009elements} has been a cornerstone of classical machine learning theory that illustrates the relationship between the bias of a model and its variance, and how they affect its generalization performance. It states that if the model is too simple (high bias), it may underfit as it does not capture the underlying patterns in the data. However, if it is too complex (high variance), it may overfit noise in the training data and fail to generalize well. The resulting conventional wisdom was to adjust the model complexity to achieve a balance between underfitting and overfitting, which would then lead to good generalization.

This classical viewpoint has been uprooted by modern practice in deep learning, where it is common to use heavily overparameterized neural networks that fit the used training data (almost) perfectly. In spite of this \mbox{(near-)perfect} fit, these models can generalize well to new data. In fact, it can be observed that as the model complexity increases, the test error first decreases, then increases (as predicted by the bias-variance trade-off), and then decreases again. This phenomenon, coined the \emph{double descent phenomenon}~\citep{Belkin_2019}, is well documented not only for deep neural networks but for a wide range of machine learning methods, see, \eg, \cite{Belkin_2019,Bel21, nakkiran2019deep,mei2020generalization, hastie2020surprises}. The second descent of the test error is observed at the \emph{interpolation threshold}, where the model has become complex enough to interpolate the training samples. Thus, to gain a deeper understanding of double descent it is important to identify at which size a neural network can interpolate finitely many samples.

To determine the interpolation threshold, we may look at the literature on the \emph{memorization capacity} of neural networks, which quantifies the number of parameters and neurons necessary for a network to be able to interpolate \emph{any} $N$ data points with \emph{arbitrary} labels. Thus, memorization capacity offers a worst-case quantitative analysis of the interpolation threshold. In this analysis, `the network architecture comes first and the data comes later'. As a result, the required network complexity for memorization scales in terms of the number of training data (see Section~\ref{sec:related-works} for more details). In practical applications, however, `the data comes first and the network architecture comes later': the neural network architecture and size are tuned to given training data via cross-validation. Intuitively, one expects that the training data possesses some `nice' structure so that interpolation is achievable with a smaller network complexity than suggested by memorization capacity---which assumes arbitrary data and arbitrary labels.

In this paper, we investigate interpolation by adopting an instance-specific viewpoint. We introduce a simple randomized algorithm that, given a fixed finite data set with two classes, with high probability constructs an interpolating neural network in polynomial time, see Theorem~\ref{thm:main-result-informal}. We then link the required number of parameters to the \emph{mutual complexity} of the data set, which depends on both the geometric properties of two data classes as well as their mutual arrangement. As a result, we obtain guarantees that are independent of the number of samples and instead yield a `problem-adaptive' bound on the interpolation threshold. Finally, we carry out numerical simulation experiments to illustrate our theoretical result. In addition, we investigate the effectiveness of our interpolation algorithm on MNIST and CIFAR-10.

\subsection{Summary of Results}
\label{sec:summary}

Let us first formalize the concept of interpolation in a classification setting with two classes. The setting with binary labels is considered for simplicity---our results can be readily extended to multiple classes by a one-versus-many approach (see Section~\ref{sec:experiments/multiclass-classification} for details). In the following, $\set{X}^-, \set{X}^+ \subset R \B_2^d$ denote disjoint and finite sets, representing two classes of objects, where $\B_2^d$ denotes the unit Euclidean ball in $\R^d$.

\begin{definition}[Interpolation]\label{prob:interpolation}
    We say that a classification function $F\colon \R^d \to \{\pm 1\}$ \emph{interpolates} $\set{X}^-$ and $\set{X}^+$ if, for all $\vec{x}^- \in \set{X}^-$ and $\vec{x}^+ \in \set{X}^+$,
    \begin{equation*}
        F(\vec{x}^-) = -1
        \quad\text{and}\quad
        F(\vec{x}^+) = +1.
    \end{equation*}
\end{definition}

In this work, we will formulate a concrete, randomized algorithm that takes $\set{X}^-$ and $\set{X}^+$ as inputs and produces an interpolating neural net as an output (Algorithm~\ref{alg:pruning}). As the statement of the algorithm requires some technical preparation, we postpone its discussion to Section~\ref{sec:mainResults}. Our main result, informally stated as  Theorem~\ref{thm:main-result-informal} below and developed in full detail in Section~\ref{sec:mainResults}, shows that this algorithm succeeds with high probability in polynomial time and provides bounds on the size of the interpolating network. The bounds are phrased in terms of two structural assumptions on the data, that together quantify the difficulty of the interpolation problem.

First, we will assume that the classes are \emph{$\delta$-separated}. This assumption is also common in a number of works on memorization capacity, \eg,~\cite{vershynin2020memory,rajput2021exponential,vardi2021optimal}. Below we will write, for any sets $\set{A}, \set{B} \subset \R^d$,
\begin{equation*}
    \dist(\vec{a},\set{B}) = \inf_{\vec{b}\in \set{B}} \|\vec{a}-\vec{b}\|_2, \qquad \dist(\set{A},\set{B}) = \inf_{\vec{a}\in \set{A}} \dist(\vec{a},\set{B}).
\end{equation*}

\begin{definition}[$\delta$-separation]\label{asu:separation}
    $\set{A}$ and $\set{B}$ are \emph{$\delta$-separated} if $\dist(\set{A},\set{B})\geq \delta$.
\end{definition}

Second, we will quantify the problem difficulty using the following notion that was first introduced in \cite{dirksen2022separation} (in a slightly different form). 

\begin{definition}[Mutual covering]\label{asu:mutual-covering}
    We call 
    \begin{equation*}
    \begin{aligned}
        \set{C}^- & = \{\vec{c}_1^-, \dots, \vec{c}_{M^-}^-\} \subset \set{X}^-, & r_1^-, \dots, r_{M^-}^- & \geq 0,\\
        \set{C}^+ & = \{\vec{c}_1^+, \dots, \vec{c}_{M^+}^+\} \subset \set{X}^+, & r_1^+, \dots, r_{M^+}^+ & \geq 0
    \end{aligned}
    \end{equation*}
    a \emph{mutual covering} for $\set{X}^-$ and $\set{X}^+$ if the sets
    \begin{equation*}
        \set{X}_\ell^- \coloneqq \set{X}^- \cap \B_2^d(\vec{c}_\ell^-, r_\ell^-)
        \quad\text{and}\quad
        \set{X}_j^+ \coloneqq \set{X}^+ \cap \B_2^d(\vec{c}_j^+, r_j^+),
    \end{equation*}
    for $\ell \in \range{M^-}$ and $j \in \range{M^+}$, cover $\set{X}^-$ and $\set{X}^+$, respectively. We call these sets the \emph{components} of the mutual covering and call $M^-$ and $M^+$ the \emph{mutual covering numbers}. 
\end{definition}

\begin{figure}[t]
    \centering
    \includegraphics[width=0.85\linewidth]{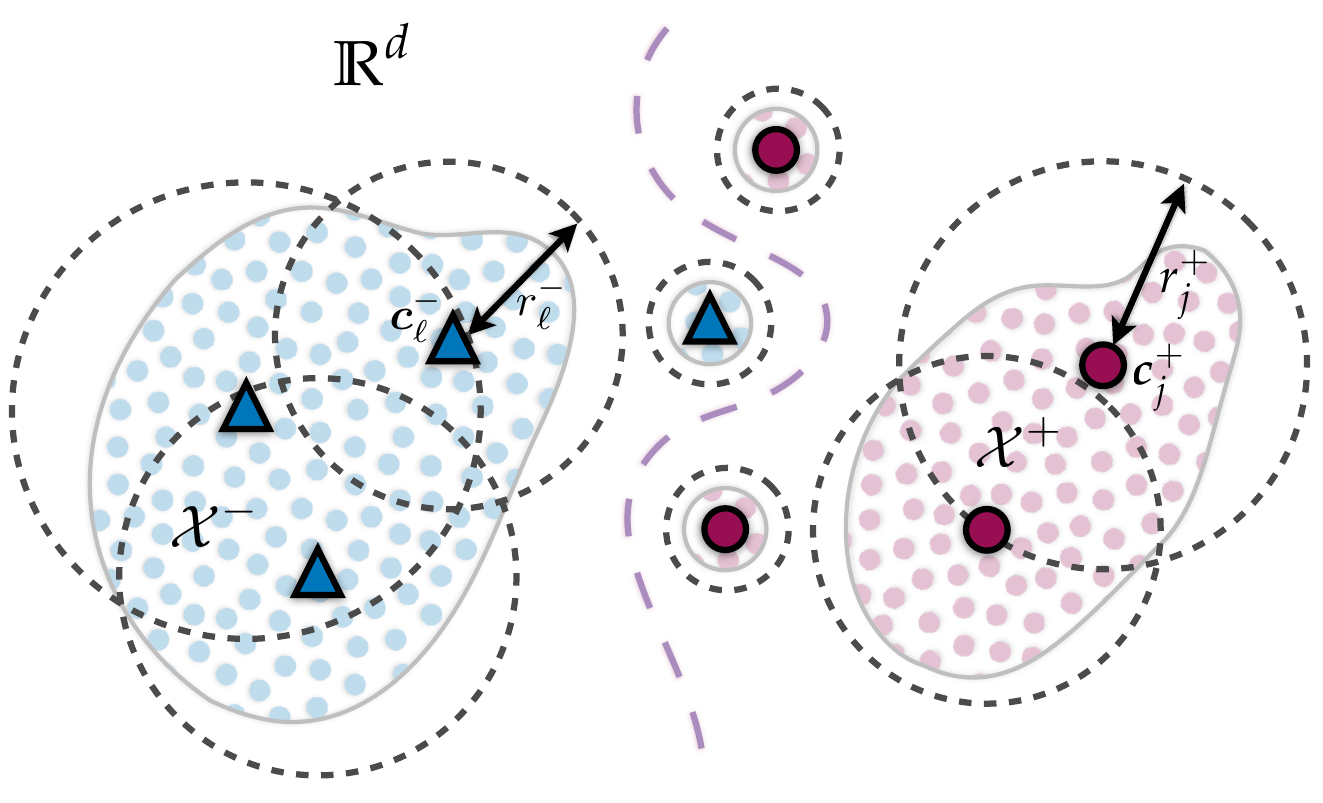}
    \caption{\textbf{The mutual covering is `problem-adaptive'.} Condition~\eqref{eq:condition-radius} on the radii in Theorem~\ref{thm:main-result-informal} allows a covering `adapted to' the mutual arrangement of the data: only the parts of the data that lie close to the ideal decision boundary need to be covered using balls with small diameters---other parts can be crudely covered using larger balls.}
    \label{fig:mutual-covering}
\end{figure}

As we have only finitely many inputs, clearly a mutual covering always exists. However, if the arrangement of the classes is benign, the mutual covering numbers can be much smaller than the number of samples. 

To see how the notion of mutual covering allows us to quantify the difficulty of a (binary) interpolation problem, we turn to our main result. In Theorem~\ref{thm:main-result-informal} we require the existence of a mutual covering with radii
\begin{equation}\label{eq:condition-radius}
    r_\ell^- \lesssim \frac{\dist(\vec{c}_\ell^-, \set{C}^+)}{\log^{1/2}(e R / \dist(\vec{c}_\ell^-, \set{C}^+))}
    \quad\text{and}\quad
    r_j^+ \lesssim \frac{\dist(\vec{c}_j^+, \set{C}^-)}{\log^{1/2}(e R / \dist(\vec{c}_j^+, \set{C}^-))}.
\end{equation}
Geometrically, this means that the components covering $\set{X}^-$ and $\set{X}^+$ cannot intersect (more precisely, need to be slightly separated from) the ideal decision boundary between the two sets, as is illustrated in Figure~\ref{fig:mutual-covering}. In particular, components with a small radius are only needed close to the ideal decision boundary, while parts that are far away from this boundary can be crudely covered with large components. Compared to classical coverings with balls of a fixed radius \citep[as used in the classical notion of the Euclidean covering number of a set, see, \eg,][]{vershynin2018high}, this can drastically reduce the required number of components.

While the mutual covering numbers $M^-$ and $M^+$ can be viewed as a measure of the \emph{global complexity} of the data, our result also involves the \emph{local complexity}, measured by the `sizes' of the components. Specifically, define $\omega \coloneqq \max\{\omega^-, \omega^+\}$ where
\begin{equation}\label{eq:local-complexity-params}
    \omega^- \coloneqq \max_{\ell \in \range{M^-}} \frac{\gw^2(\set{X}_\ell^- - \vec{c}_\ell^-)}{\dist^3(\vec{c}_\ell^-, \set{C}^+)}
    \quad\text{and}\quad
    \omega^+ \coloneqq \max_{j \in \range{M^+}} \frac{\gw^2(\set{X}_j^+ - \vec{c}_j^+)}{\dist^3(\vec{c}_j^+, \set{C}^-)}.
\end{equation}
The quantities $\omega^-$ and $\omega^+$ measure the scaled version of the `size' of the largest (centered) component of $\set{X}^-$ and $\set{X}^+$, respectively. Here, the \emph{Gaussian mean width} of a set $\set{A} \subset \R^d$ is defined as
\begin{equation*}
    \gw(\set{A}) \coloneqq \E \sup_{\vec{x} \in \set{A}} | \langle \vec{g}, \vec{x} \rangle | ,
\end{equation*}
where $\vec{g} \sim \Normal(\vec{0}, \mat{I}_d)$ denotes a standard Gaussian random vector. The mean width is a well-established complexity measure in high-dimensional statistics and geometry which is sensitive to low-dimensional structures such as sparsity, unions of low-dimensional subspaces, or manifolds, see, \eg, \cite{vershynin2018high} for a detailed discussion and examples. We refer to Remark~\ref{rem:omegaEst} for straightforward estimates of $\omega$.

We are now ready to present the informal version of our main result, which we state for the case of threshold activations. Intuitively, this should be the most challenging because the signal amplitude is lost. It is possible to prove analogous results for other activations. In fact, Algorithm~\ref{alg:pruning} only requires an activation function $\sigma$ such that $\sigma(t) = 0$ for $t \leq 0$ and $\sigma(t) > 0$ for $t > 0$, which, \eg, includes the ReLU. Note, however, that the bounds on the network size may change for different activations.

\begin{theorem}[Informal]\label{thm:main-result-informal}
    Let $\set{X}^-, \set{X}^+ \subset R\B_2^d$ be finite and disjoint. Suppose that there is a mutual covering with $\delta$-separated centers and radii satisfying \eqref{eq:condition-radius}. Then, with high probability, Algorithm~\ref{alg:pruning} terminates in polynomial time and outputs a $2$-hidden-layer fully-connected neural network with threshold activations,
    \begin{equation*}
        \O\left( M^- + R\delta^{-1} \log(2 M^- M^+) + R \omega \right)
    \end{equation*}
    neurons and
    \begin{equation*}
        \O\left( R (d + M^-) (\delta^{-1} \log(2 M^- M^+) + \omega ) \right)
    \end{equation*}
    parameters, that interpolates $\set{X}^-$ and $\set{X}^+$.
\end{theorem}

A first interesting feature of this result is the asymmetric dependence on the complexities of the classes: the network size depends linearly on $M^-$ but only logarithmically on $M^+$. As it is possible to interchange the roles of $\set{X}^-$ and $\set{X}^+$, we may always think of $\set{X}^-$ as the `smaller' set. Second, our bounds are independent of the number of samples. This is a fundamental difference between memorization capacity and our instance-specific approach to interpolation, see the discussion in Section~\ref{sec:related-works}. To highlight this second point further, we deduce an interpolation result for infinite sets from our analysis. In contrast to Theorem~\ref{thm:main-result-informal}, the proof is nonconstructive.

\begin{corollary}
    \label{cor:support-cover-informal}
    Let $\set{X}^-, \set{X}^+ \subset R \B_2^d$ be (possibly infinite) sets. Suppose that there is a mutual covering with $\delta$-separated centers and radii satisfying~\eqref{eq:condition-radius}. Then, there exists a neural network of the same size as in Theorem~\ref{thm:main-result-informal} that interpolates $\set{X}^-$ and $\set{X}^+$.
\end{corollary}

\subsection{Organization}

The rest of the paper is organized as follows. In Section~\ref{sec:related-works} we discuss related works, then introduce notation in Section~\ref{sec:notation}. In Section~\ref{sec:mainResults} we present Algorithm~\ref{alg:pruning} and give intuition on how it works. We also state the formal counterpart of our main result in Theorem~\ref{thm:pruning/covering-bound}. All proofs are contained in Section~\ref{sec:proofs}. In Section~\ref{sec:numerical-experiments}, we verify our theoretical findings in illustrative numerical experiments and additionally investigate the performance of our algorithm on real data sets. We conclude with a short summary in Section~\ref{sec:conclusion}.

\subsection{Related Works}
\label{sec:related-works}

\paragraph{Memorization capacity.}

Neural network architectures used in practice are powerful memorizers: it has been observed that various popular architectures for image classification do not only interpolate their training data, but can even interpolate this data when the labels are replaced by random labels (after re-training), see, \eg, \citet{zhang2017understanding}. To understand this phenomenon, an extensive literature has studied the memorization capacity of neural networks, by quantifying how large a network needs to be to interpolate any $N$ points with \emph{arbitrary labels}. In this case, we will say that the network can memorize $N$ points. In practice, memorization results often include some assumptions on the inputs. Here we will summarize relevant memorization literature that makes similar structural assumptions on the inputs, such as $\delta$-separation or a bound on the norm. Other works consider randomized samples or samples drawn from a distribution, see, \eg, \citet{GWZ19,Dan20,ZZH21}.

The study of the memorization capacity of neural networks with threshold activations has a rich history. Assuming that the points are in general position,\footnote{A set of $N$ points in $\R^d$ is said to be in general position if any subset of $d$ vectors is linearly independent.} \cite{baum1988capabilities} showed that a 1-hidden-layer threshold network with $\O(N + d)$ parameters and $\O(\lceil N / d \rceil)$ neurons is enough to memorize binary labels of $N$ points in $\R^d$. In \cite{huang1990bounds} it was shown that $\O(Nd)$ parameters and $\O(N)$ neurons are enough to memorize real labels, without placing any additional constraints on the points. Assuming that the points are $\delta$-separated and lie on the unit sphere, \cite{vershynin2020memory} proved that a deep threshold (or ReLU) network can memorize binary labels using $\widetilde{\O}(e^{1/\delta^2} (d + \sqrt{N}) + N)$ parameters and $\widetilde{\O}(e^{1/\delta^2} + \sqrt{N})$ neurons. The exponential dependence on $\delta$ was improved by \cite{rajput2021exponential}, who proved that $\widetilde{\O}(d/\delta + N)$ parameters and $\widetilde{\O}(1/\delta + \sqrt{N})$ neurons are enough for memorization of binary labels, while further only requiring bounded norm instead of unit norm. The constructions of both \cite{vershynin2020memory} and \cite{rajput2021exponential} are probabilistic, while the ones of \cite{baum1988capabilities} and  \cite{huang1990bounds} are purely deterministic.

There have been a number of works on the memorization capacity of networks with other activations. We will only summarize the results for ReLU activations due to its popularity in practice, and refer to, \eg, \cite{Hua03,park2021provable,MaT23} and the references therein for other activations. The work \cite{bubeck2020network} extended the result of \cite{baum1988capabilities} to the case of real-valued labels using a network with ReLU activation with a size of the same order. Using weight sharing in the first layer, \cite{zhang2017understanding} showed that a 1-hidden-layer ReLU network could memorize real-valued labels using $\O(N + d)$ parameters and $\O(N)$ neurons, with no further assumptions on the points. \cite{yun2019small} proved that both multi-class and real-valued labels can be memorized by a ReLU net with two and three hidden layers, respectively, using $\O(d \sqrt{N} + N)$ parameters and $\O(\sqrt{N})$ neurons. \cite{park2021provable} achieved the first result on memorization with a sublinear number of parameters: assuming that the points are separated, they showed that ReLU (or hard-tanh) nets can memorize multiple classes using $\widetilde{\O}(d + N^{2/3})$ parameters, constant width and $\widetilde{\O}(N^{2/3})$ layers. \cite{vardi2021optimal} improved the above dependence on $N$ from $N^{2/3}$ to $\sqrt{N}$, which is optimal. Specifically, assuming that the points are $\delta$-separated and have bounded norm, they show that a ReLU net with $\widetilde{\O}(d + \sqrt{N})$ parameters, constant width and $\widetilde{\O}(\sqrt{N})$ layers is enough to memorize multi-class labels.

To directly compare the above with our results, we consider a \emph{trivial} mutual covering that always `works' regardless of the labels of the points: we cover each point by its own component with a radius of zero. Thus, $M^- = N^- \coloneqq |\set{X}^-|$, $M^+ = N^+ \coloneqq |\set{X}^+|$ and $\omega = 0$. Hence, in the worst case Theorem~\ref{thm:main-result-informal} yields a network with $\O(R (d + N^-) \delta^{-1} \log(2 N^- N^+))$ parameters and $\O\left(N^- + R\delta^{-1} \log(2 N^- N^+)\right)$ neurons. If $N^- \simeq N^+$, the number of neurons scales (slightly worse than) linear in the number of points, which is worse than the best result on memorization capacity for networks using the threshold activation. In Proposition~\ref{prop:pruning/sharpness-of-upper-bound} we show that the linear scaling in terms of $M^-$ in Theorem~\ref{thm:main-result-informal} is not a proof artifact. Hence, our method cannot recover optimal performance in the worst case. It is an interesting open question whether our method can be modified to achieve this. 

Nevertheless, in practical situations one can hope that a much better mutual covering exists, due to intrinsic low-dimensional structure of the input data and/or a more benign label assignment than arbitrary labelling. In such cases Theorem~\ref{thm:main-result-informal} can guarantee a much smaller interpolating network. In particular, since our bounds are independent of the number of samples we can derive interpolation results for infinite sets (Corollary~\ref{cor:support-cover-informal}). In contrast, results on memorization capacity cannot have this feature. The VC-dimension\footnote{The VC-dimension is the maximal $N$ for which there exist points $\vec{x}_1, \dots, \vec{x}_N \in \R^d$ such that for every assignment of labels $y_1, \dots, y_N \in \{\pm 1\}$ there exists a set of parameters $\theta$ such that the network interpolates the samples, \ie, $F_\theta(x_i) = y_i$ for all $i \in \range{N}$.} of feed-forward neural networks with threshold activation is $\O(W \log W)$ \citep{Baum1989WhatSN}, where $W$ denotes the total number of parameters, \ie, the sum of the number of weights and biases over all layers. Hence, to memorize more samples than this upper bound, one would necessarily need to add more parameters to the network. Similar results hold for arbitrary piecewise linear activations such as the ReLU~\citep{bartlett2017nearlytight} or analytic definable activation functions~\citep{sontag1997shattering}.

Upon acceptance of this paper, the work of \cite{lee2024defining} was pointed out to us by one of the reviewers, which bears some similarities to ours. It introduces the concept of a polytope-basis cover of a dataset of two classes. They show that if this basis is known, then an interpolating three-layer fully-connected ReLU network can be associated to such a cover. They then provide upper (and some lower) bound on the network width sufficient for the existence of an interpolating net if the data is a convex polytope, structured as a simplicial complex, or can be covered by the difference of prismatic polytopes. While their theoretical guarantees are pure existence results, they also introduce a number of heuristic algorithms that yield small near-interpolating ReLU nets on, e.g., MNIST and CIFAR10. These methods are only guaranteed to terminate and, in contrast to our work, no guarantees are derived on the runtime, size of the network, and interpolation success.

\paragraph{Separation capacity.}

Related to interpolation is the question of \emph{separation capacity} of a neural network: under what conditions can a neural network make two (not necessarily finite) classes linearly separable? Obviously, a network with separation capacity can be extended to an interpolating network by adding the separating hyperplane as an additional layer.

In \cite{an2015can} it was shown that any two disjoint sets can be made linearly separable using a deterministic two-layer ReLU neural net. However, their proof is non-constructive and they provided no estimates on the size of the network. Inspired by this work, \cite{dirksen2022separation} showed that a wide enough two-layer random ReLU network can make any two $\delta$-separated sets linearly separable if the weights and biases are chosen from appropriate distributions. Unlike the existence result of \cite{an2015can}, they provided bounds linking the number of required neurons to geometric properties of the classes and their mutual arrangement via a notion of mutual covering similar to Definition~\ref{asu:mutual-covering}. This instance-specific viewpoint allows them to overcome the curse of dimensionality if the data carries a low-complexity structure. Following up on this, \cite{ghosal2022randomly} showed that even a wide enough one-layer ReLU net is enough to accomplish separation. They introduced a deterministic memorization algorithm which is then `implemented' by a random neural network. As \cite{dirksen2022separation} they also used a mutual covering to capture the complexity of the data.

While the above results could be applied to interpolation, the required number of parameters would be larger than what we require in Theorem~\ref{thm:main-result-informal}. Both \cite{dirksen2022separation} and \cite{ghosal2022randomly} yield networks scaling polynomially in terms of the mutual covering numbers, while our network scales only linearly.

The present paper is strongly influenced by \cite{dirksen2022separation}---we adopt an instance-specific viewpoint and the notion of mutual covering. However, instead of separation, we directly focus on interpolation. Together with our only partially randomized approach, this allows us to prove better bounds for this case.

\paragraph{Random hyperplane tesselations.}

As will become apparent below, our technical analysis is linked to tessellations created by random hyperplanes with Gaussian directions and uniformly distributed shifts, which were recently intensively studied in \cite{DiM21,DMS22}. In particular, \cite{DMS22} derived a sharp bound on the number of hyperplanes needed to induce a uniform tessellation of a given set, meaning that the Euclidean distance between any two points in the set corresponds to the fraction of hyperplanes separating them up to a prespecified error. We will use some insights from these works, see in particular Lemma~\ref{lem:signPert}.

\subsection{Setup and Notation}
\label{sec:notation}

For any $1\leq p\leq \infty$ we let $\norm{\cdot}_p$ denote the $\ell_p$ norm. We use $\B_2^d(\vec{c}, r)$ to denote the Euclidean ball in $\R^d$ with center $\vec{c} \in \R^d$ and radius $r \geq 0$ and we denote the unit ball by $\B_2^d$. For $n \in \N$, we set $\range{n} \coloneqq \{1, \dots, n\}$. For any set $\set{A}$ we use $|\set{A}|$ to denote its cardinality and let $\1_{\set{A}}$ denote its indicator. We let $\sign$ denote the function
\begin{equation*}
    \sign(x)=\begin{cases}
    +1 & \text{if } x\geq 0,\\
    -1 & \text{else}.
    \end{cases}
\end{equation*}
For a function $\sigma \colon \R \to \R$ and a vector $\vec{x} \in \R^d$ we denote the element-wise application by $\sigma(\vec{x}) = (\sigma(x_i))_{i=1}^d$. If an equality holds up to an absolute constant $C$, we write $A \gtrsim B$ instead of $A \geq C \cdot B$. We write $A \simeq B$ if $A \gtrsim B \gtrsim A$. We use $\O(\, \cdot \,)$ to omit constant terms and $\widetilde{\O}(\, \cdot \,)$ to additionally omit logarithmic terms. We define the distance between any point $\vec{x} \in \R^d$ and a set $\set{X} \subset \R^d$ as $\dist(\vec{x}, \set{X}) \coloneqq \inf \{ \norm{\vec{x} - \vec{y}}_{2} : \vec{y} \in \set{X} \}$. We denote the hyperplane with direction $\vec{v} \in \R^d$ and shift $\tau \in \R$ by $H[\vec{v}, \tau] \coloneqq \{ \vec{x} \in \R^d : \langle \vec{v}, \vec{x} \rangle + \tau = 0 \}.$ For $\vec{x}, \vec{y} \in \R^n$ we define $\1[\vec{x}=\vec{y}] \in \{0, 1\}^n$ by
\begin{equation*}
    (\1[\vec{x} = \vec{y}])_i = \begin{cases}
        1 & \text{if $x_i = y_i$},\\
        0 & \text{else}.
    \end{cases}
\end{equation*}
We denote by $\vec{0}, \vec{1} \in \R^d$ the vector with entries all equal to $0$ and all equal to $1$, respectively. We denote the standard multivariate normal distribution in $d$ dimensions by $\Normal(\vec{0}, \mat{I}_d)$ and the uniform distribution on $\set{A} \subset \R^d$ by $\Unif(\set{A})$.

\section{Interpolation Algorithm and Main Results}
\label{sec:mainResults}

\begin{figure}[t]
    \centering
    \includegraphics[width=.7\linewidth]{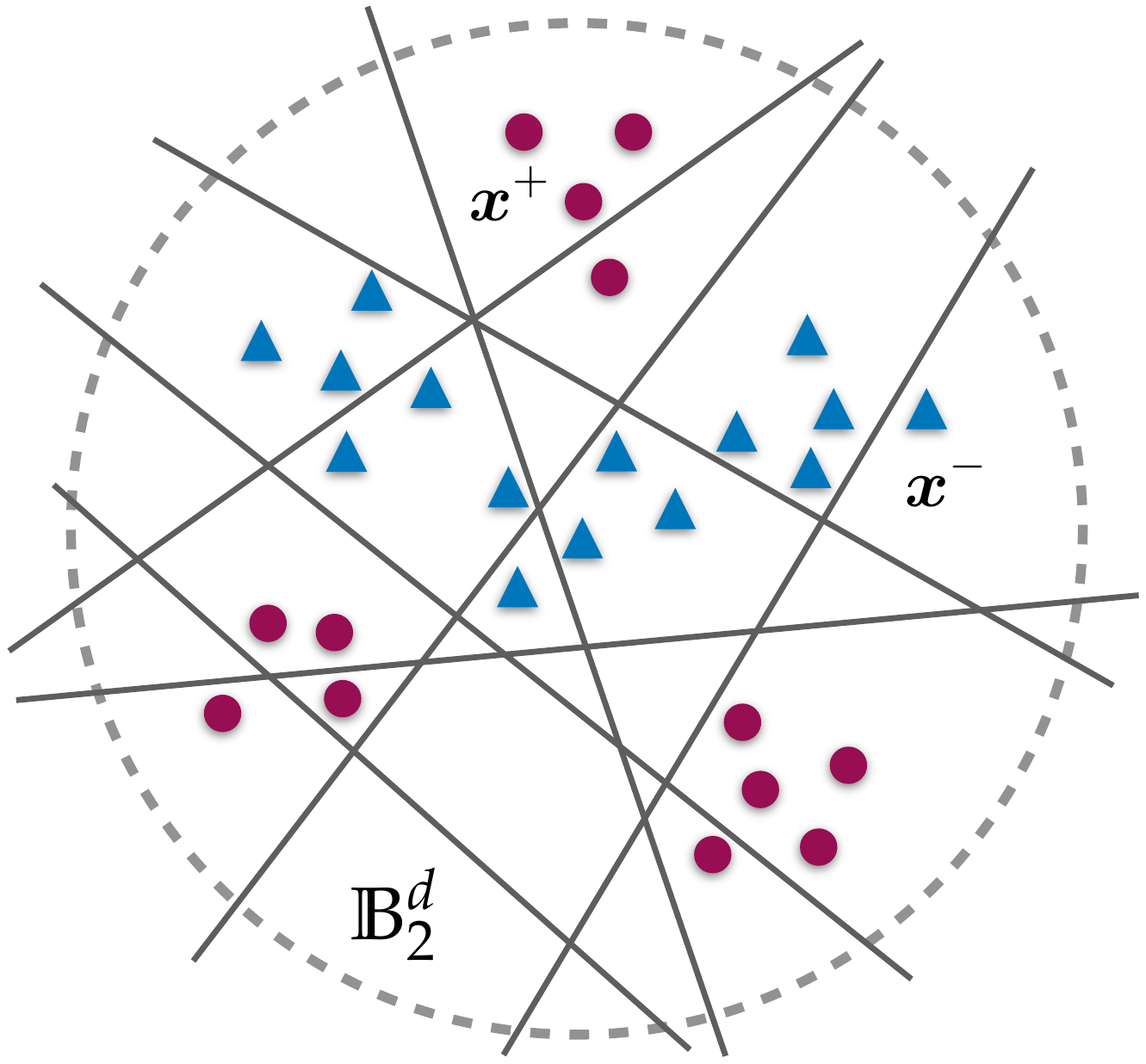}
    \caption{\textbf{Random hyperplanes in the input domain $\R^d$.} In Algorithm~\ref{alg:pruning} we iteratively sample random hyperplanes $H[\vec{w}_i, b_i]$ until every pair of points with opposite labels is separated by at least one of them. This tessellates the space into multiple cells, where each cell is only populated with points of the same label. Each hyperplane can be associated with one of the neurons of the first layer $\Phi$.}
    \label{fig:interpolation/tesselation}
\end{figure}

\begin{figure}[t]
    \centering
    \includegraphics[width=.75\linewidth]{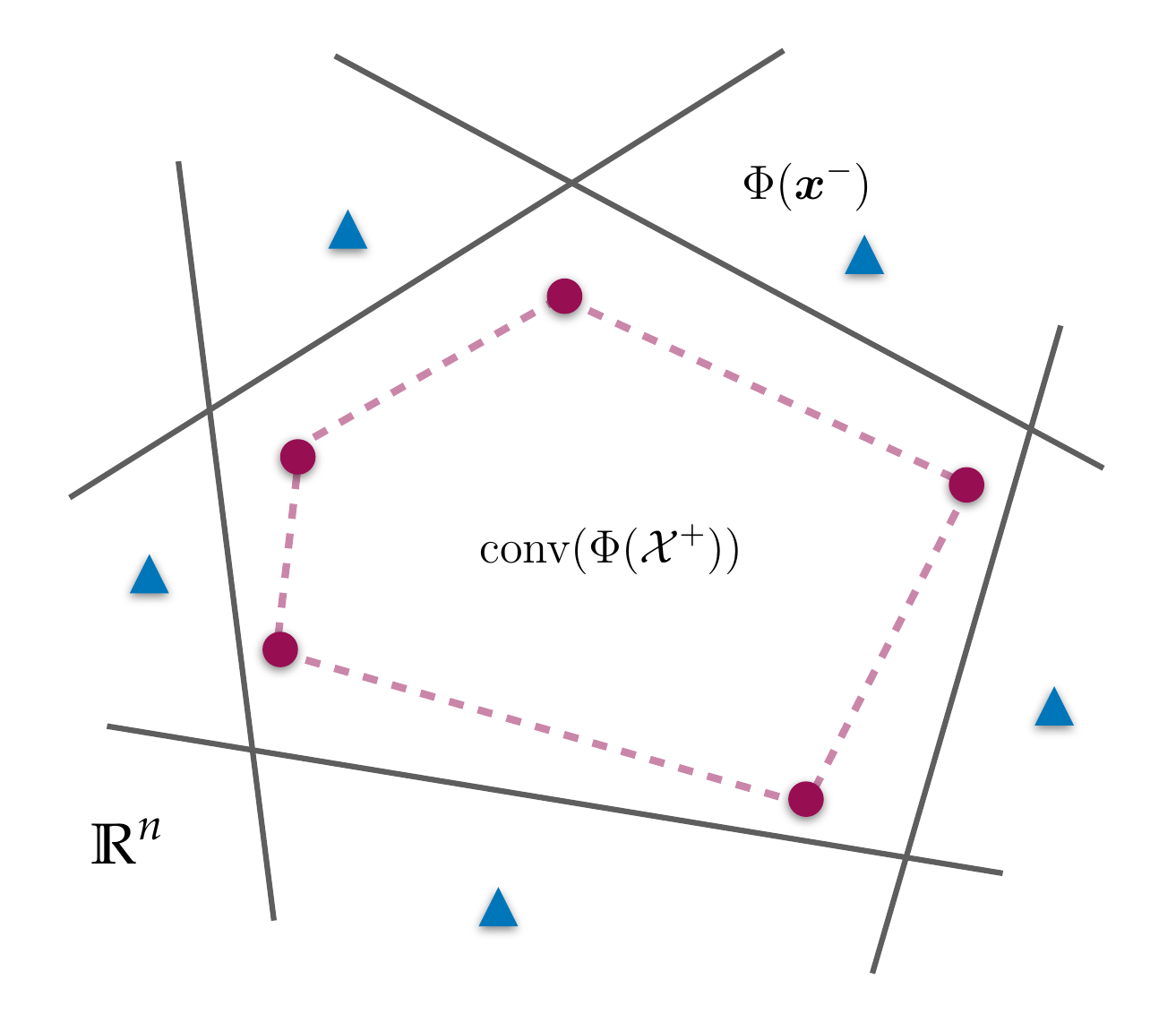}
    \caption{\textbf{The effect of the first layer $\Phi$.} After transforming the data with the first layer $\Phi$ we can, for each $\vec{x}^- \in \set{X}^-$, construct a hyperplane $H[-\vec{u}_{\vec{x}^-}, m_{\vec{x}^-}]$ that separates $\Phi(\set{X}^+)$ from $\Phi(\vec{x}^-)$. Each hyperplane can be associated with one of the neurons in the second layer.}
    \label{fig:interpolation/first-layer}
\end{figure}

Consider any disjoint $\set{X}^-, \set{X}^+ \subset R \B_2^d$ with $N^- \coloneqq |\set{X}^-|$ and $N^+ \coloneqq |\set{X}^+|$. Let $\sigma\colon \R \to \R$ satisfy $\sigma(t) = 0$ for $t \leq 0$ and $\sigma(t) > 0$ for $t > 0$. Let us outline our method to construct an interpolating three-layer neural network:

\begin{enumerate}
    \item To build the first layer $\Phi\colon \R^d \to \R^n$, we iteratively sample i.i.d. random hyperplanes $H[\vec{w}_i, b_i]$ until any $\vec{x}^- \in \set{X}^-$ is separated from any $\vec{x}^+ \in \set{X}^+$ by at least one of them (see Figure~\ref{fig:interpolation/tesselation} and Definition~\ref{def:hyperplane-separation}). Each hyperplane includes a shift $b_i$ so that it is able to separate points located on a ray emanating from the origin. In the worst case, one could have points with opposite labels close to the boundary of $R \B_2^d$, hence one needs the maximal shift to scale at least like $R$. We let $\mat{W}$ be the matrix containing the $\vec{w}_i$ as its rows and let $\vec{b}$ be the vector having the $b_i$ as its coordinates. We define the first, random layer $\Phi$ of the network by $\Phi(\vec{x})=\sigma(\mat{W}\vec{x}+\vec{b})$. Since all pairs of points with opposite labels are separated by at least one hyperplane, $\Phi$ has the following property: for any $(\vec{x}^-, \vec{x}^+) \in \set{X}^- \times \set{X}^+$ there exists at least one $i \in \range{n}$ with
    \begin{equation}\label{eq:interpolation/motivation/first-layer}
        \Phi_i(\vec{x}^-) = 0
        \quad\text{and}\quad
        \Phi_i(\vec{x}^+) > 0.
    \end{equation}
    This enables us to distinguish between points of different labels.
    
    \item We then exploit~\eqref{eq:interpolation/motivation/first-layer} in the following way. For $\vec{x}^- \in \set{X}^-$ consider the mask $\vec{u}_{\vec{x}^-} = \1[\Phi(\vec{x}^-) = \vec{0}]$. By~\eqref{eq:interpolation/motivation/first-layer},
    \begin{equation*}
        \langle \vec{u}_{\vec{x}^-}, \Phi(\vec{x}^-) \rangle = 0
        \quad\text{and}\quad
        \langle \vec{u}_{\vec{x}^-}, \Phi(\vec{x}^+) \rangle > 0 \quad \text{for all } \vec{x}^+ \in \set{X}^+.
    \end{equation*}
    Geometrically, this means that the hyperplane $H[-\vec{u}_{\vec{x}^-}, m_{\vec{x}^-}]$, where
    \begin{equation*}
        m_{\vec{x}^-} = \min_{\vec{x}^+ \in \set{X}^+} \langle \vec{u}_{\vec{x}^-}, \Phi(\vec{x}^+) \rangle,
    \end{equation*}
    separates $\Phi(\set{X}^+)$ from $\Phi(\vec{x}^-)$ (see Figure~\ref{fig:interpolation/first-layer}). Let $\mat{U} \in \R^{N^- \times n}$ be the matrix with rows $\vec{u}_{\vec{x}^-}$ and let $\vec{m} \in \R^{N^-}$ be the vector with coordinates $m_{\vec{x}^-}$. We then define the second layer $\hat{\Phi}\colon \R^n \to \R^{\hat{n}}$ of the network by $\hat{\Phi}(\vec{z}) = \sigma(- \mat{U} \vec{z} + \vec{m})$. This layer satisfies, for every $\vec{x}^-\in \set{X}^-$,
    \begin{equation}\label{eq:interpolation/motivation/second-layer}
        [\hat{\Phi}(\Phi(\vec{x}^-))]_{\vec{x}^-} > 0
        \quad\text{and}\quad
        [\hat{\Phi}(\Phi(\vec{x}^+))]_{\vec{x}^-} = 0 \quad \text{for all } \vec{x}^+ \in \set{X}^+.
    \end{equation}
    Thus, in the second hidden layer, there is a dedicated neuron to detect each point of $\set{X}^-$, but none of them activates on $\set{X}^+$.
    
    \item In the output layer, we simply sum the output from the second layer $\hat{\Phi}$. By~\eqref{eq:interpolation/motivation/second-layer}, for all $\vec{x}^- \in \set{X}^-$ and $\vec{x}^+ \in \set{X}^+$,
    \begin{equation*}
        \langle \vec{1}, \hat{\Phi}(\Phi(\vec{x}^-)) \rangle > 0
        \quad\text{and}\quad
        \langle \vec{1}, \hat{\Phi}(\Phi(\vec{x}^+)) \rangle = 0
    \end{equation*}
    and hence  $\sign(-\cdot)$ outputs the correct label.
\end{enumerate}

The second step of this method is rather naive: for \emph{every} $\vec{x}^-\in \set{X}^-$, we construct a dedicated neuron 
\begin{equation}\label{eqn:dedNeur}
    \hat{\varphi}_{\vec{x}^-}(\vec{z}) = \sigma(-\langle \vec{u}_{\vec{x}^-}, \vec{z} \rangle + m_{\vec{x}^-})
\end{equation}
that distinguishes $\Phi(\vec{x}^-)$ and $\Phi(\set{X}^+)$, \ie, $\hat{\varphi}_{\vec{x}^-}(\Phi(\vec{x}^-))>0$ and $\hat{\varphi}_{\vec{x}^-}(\Phi(\vec{x}^+))=0$ for all $\vec{x}^+\in\set{X}^+$. This potentially leads to redundancy, since to get an interpolating net at the third step, it suffices if for each $\vec{x}^-$ there is \emph{some} $\vec{x}_*^-$ such that $\hat{\varphi}_{\vec{x}_*^-}$ distinguishes $\Phi(\vec{x}^-)$ and $\Phi(\set{X}^+)$. We can especially hope for this to be true if $\vec{x}^-$ is `close enough to' $\vec{x}_*^-$ in a suitable sense. This is illustrated in Figure~\ref{fig:pruning/motivation}. Thus we can improve the second step by forward selection: we iteratively select elements $\vec{x}_*^-$ from $\set{X}^-$ and construct the associated neuron $\hat{\varphi}_{\vec{x}_*^-}$ until there is a distinguishing neuron for each element in $\set{X}^-$.

\begin{figure}
    \centering
    \includegraphics[width=.65\linewidth]{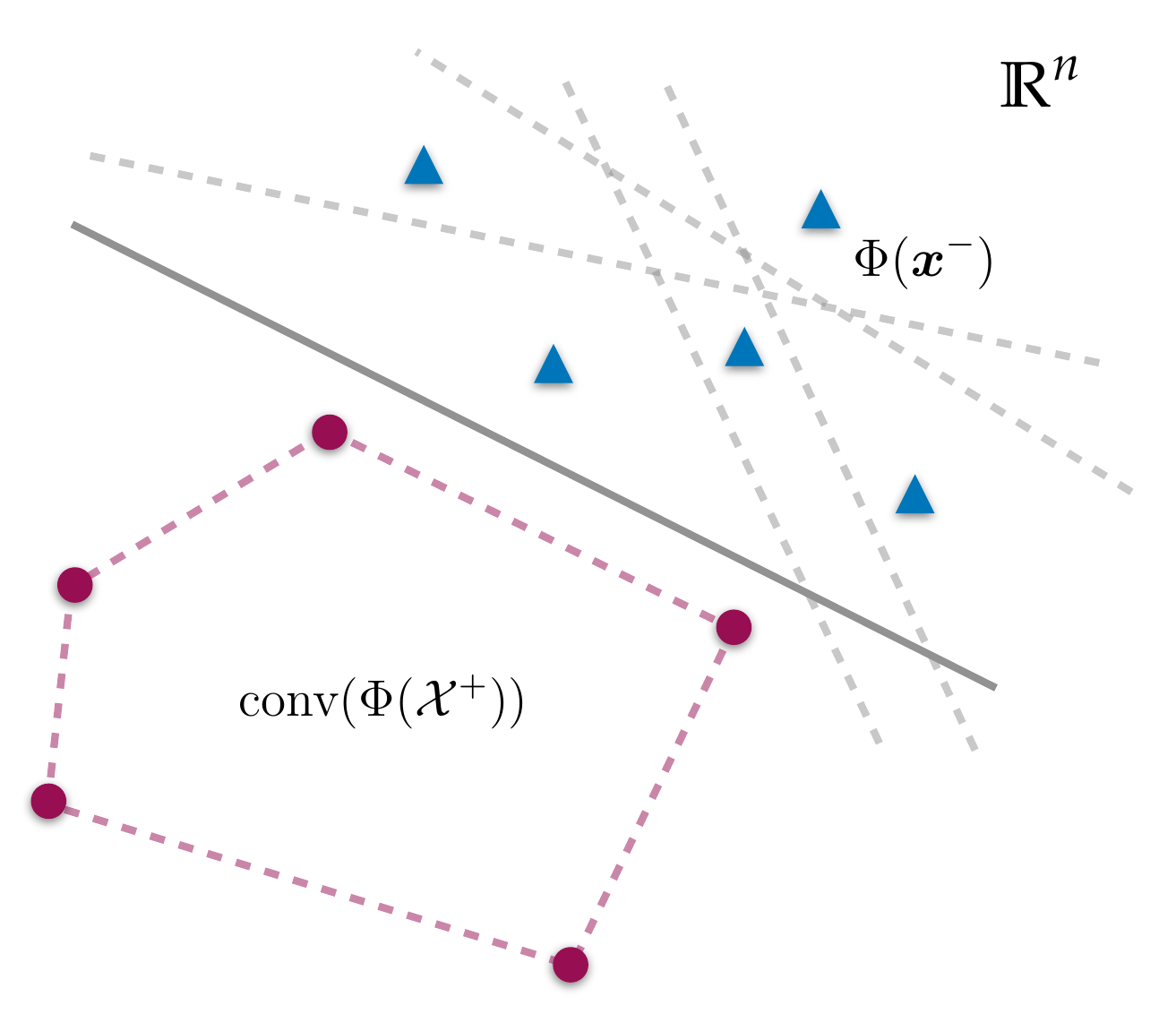}
    \caption{\textbf{Motivation for forward selection.} While each $\Phi(\vec{x}^-)$ is separated by a corresponding `dedicated' hyperplane from $\Phi(\set{X}^+)$ (depicted in dashed grey), we can identify a single hyperplane $H[-\vec{u}_{\vec{x}_*^-}, m_{\vec{x}_*^-}]$ (depicted in grey) that separates several $\Phi(\vec{x}^-)$ from $\Phi(\set{X}^+)$ simultaneously. The other hyperplanes are redundant and the corresponding neurons do not need to be included in the second layer $\hat{\Phi}$.}
    \label{fig:pruning/motivation}
\end{figure}

These considerations lead to our interpolation algorithm formalized in Algorithm~\ref{alg:pruning}.

\begin{algorithm}
\caption{Interpolation}\label{alg:pruning}
\begin{algorithmic}[1]
\Require Disjoint and finite $\set{X}^-, \set{X}^+ \subset \R^d$, activation $\sigma\colon \R \to \R$ satisfying $\sigma(t) = 0$ for $t \leq 0$ and $\sigma(t) > 0$ for $t > 0$, (minimal) width of the first layer $n_{\min} \geq 0$.
\Ensure A three-layer fully-connected neural network $F\colon \R^d \to \{\pm 1\}$ that interpolates $\set{X}^-$ and $\set{X}^+$.
\vspace{.5em}
\phase{First layer $\Phi$}
\State Calculate $R \gets \max_{\vec{x} \in \set{X}^- \cup \set{X}^+} \norm{\vec{x}}_2$ and choose $\lambda \gtrsim R$.
\State Initialize $\set{S} \gets \emptyset$ and $n \gets 0$.
\While{$\set{S} \not= \set{X}^- \times \set{X}^+$ \textbf{or} $n < n_{\min}$}\label{lin:condWhileAlgGreedy}
    \State Update $n \gets n + 1$.
    \State Sample 
    \begin{equation*}
        \vec{w}_n \sim \Normal(\vec{0}, \mat{I}_d),
        \quad
        b_n \sim \Unif([-\lambda, \lambda]).
    \end{equation*}
    \State Update $\set{S}$ according to
    \begin{equation*}
        \set{S} \gets \set{S} \cup \{ (\vec{x}^-, \vec{x}^+) \in \set{X}^- \times \set{X}^+ : \langle \vec{w}_n, \vec{x}^- \rangle \leq -b_n < \langle \vec{w}_n, \vec{x}^+ \rangle \}.
    \end{equation*}
\EndWhile
\State Define $\Phi(\vec{x}) = \sigma(\mat{W} \vec{x} + \vec{b})$ with $\mat{W} \in \R^{n \times d}$ and $\vec{b} \in \R^n$ where
\begin{equation*}
    \mat{W} \gets \begin{bmatrix} \vec{w}_1^\top\\ \vdots\\ \vec{w}_n^\top \end{bmatrix}\quad\text{and}\quad\vec{b} \gets \begin{bmatrix} b_1\\ \vdots\\ b_n \end{bmatrix}.
\end{equation*}
\phase{Second layer $\hat{\Phi}$}
\State Initialize $\set{C} \gets \set{X}^-$ and $\hat{n} \gets 0$.
\While{$\set{C} \not= \emptyset$}
    \State Update $\hat{n} \gets \hat{n} + 1$.
    \State Select $\vec{x}_{\hat{n}}^- \in \set{C}$ uniformly at random from $\set{C}$ and calculate
    \begin{equation*}
        \vec{u}_{\hat{n}} \gets \1[\Phi(\vec{x}_{\hat{n}}^-) = \vec{0}],
        \quad
        m_{\hat{n}} \gets \min_{\vec{x}^+ \in \set{X}^+} \langle \vec{u}_{\hat{n}}, \Phi(\vec{x}^+) \rangle.
    \end{equation*}
    \State Update $\set{C}$ according to
    \begin{equation*}
        \set{C} \gets \set{C} \setminus \{ \vec{x}^- \in \set{C} : \langle \vec{u}_{\vec{x}_{\hat{n}}^-}, \Phi(\vec{x}^-) \rangle < m_{\vec{x}_{\hat{n}}^-} \}.
    \end{equation*}
\EndWhile
\State Define $\hat{\Phi}(\vec{z}) = \sigma(-\mat{U} \vec{z} + \vec{m})$ with $\mat{U} \in \R^{\hat{n} \times n}$ and $\vec{m} \in \R^{\hat{n}}$ where
\begin{equation*}
    \mat{U} \gets \begin{bmatrix} \vec{u}_1^\top\\ \vdots\\ \vec{u}_{\hat{n}}^\top \end{bmatrix}
    \quad\text{and}\quad
    \vec{m} \gets \begin{bmatrix} m_1\\ \vdots\\ m_{\hat{n}} \end{bmatrix}.
\end{equation*}
\phase{Output network $F$}
\State Return $F(\vec{x}) = \sign(-\langle \vec{1}, \hat{\Phi}(\Phi(\vec{x})) \rangle)$. 
\end{algorithmic}
\end{algorithm}

\begin{remark}
    First, let us briefly comment on the parameter $n_{\min}$ in the first loop of Algorithm~\ref{alg:pruning}, which is the minimal width of the first layer $\Phi$. In (the proof of) Proposition~\ref{pro:pruning/correctness} we will see that the first loop (and hence, the algorithm) terminates with probability $1$, regardless of the choice of $n_{\text{min}}$. In Theorem~\ref{thm:pruning/covering-bound}, we will derive a lower bound on $n_{\min}$ that ensures that the algorithm terminates with high probability after $n_{\min}$ iterations and derive an upper bound on the total size of the output net $F$. The first condition in line~3 of the algorithm will in this case be redundant. We include this condition to ensure that the algorithm is always guaranteed to terminate, for any choice of $n_{\text{min}}$.
    
    Second, we comment on the parameter $\lambda$, which is the maximal shift of the hyperplanes in the first layer. The condition $\lambda \gtrsim R$ in the first line of Algorithm~\ref{alg:pruning} is used to guarantee that every pair of samples with different labels is separated by at least one of the hyperplanes (even if they are on a line through the origin, Proposition~\ref{pro:interpolation/termination}), and that they induce a uniform tesselation, allowing us to relate the fraction of hyperplanes between points to their Euclidean distance (Lemmas~\ref{lem:Thres_NN_finite_1st_layer} and~\ref{lem:signPert}). As this condition involves an unknown constant, for a practical application $\lambda$ can be treated like a hyperparameter. In Section~\ref{sec:numerical-experiments} we will see that $\lambda \geq R$ typically is sufficient and, depending on the data set, smaller values might also work.
\end{remark}

Let us now state our main results.

\begin{proposition}[Termination and correctness]
    \label{pro:pruning/correctness}
    Let $\set{X}^-, \set{X}^+ \subset \R^d$ be disjoint and finite. Then Algorithm~\ref{alg:pruning} terminates with probability $1$ and its output $F$ interpolates $\set{X}^-$ and $\set{X}^+$.
\end{proposition}

From the discussion at the start of this section, it is clear that Algorithm~\ref{alg:pruning} produces an interpolating network \emph{if} the first loop of the algorithm terminates. We will prove termination in Section~\ref{sec:proofCorrectness}.

Additionally, the following gives an estimate of the run time of Algorithm~\ref{alg:pruning}.

\begin{proposition}[Run time]
    \label{pro:run-time}
    Let $\set{X}^-, \set{X}^+ \subset \R^d$ be finite and $\delta$-separated. Let $N^- \coloneqq |\set{X}^-|$ and $N^+ \coloneqq |\set{X}^+|$ and denote $N \coloneqq N^- + N^+$. Assume that $N^- \simeq N^+$, the input dimension $d$ is constant and the activation function $\sigma$ is computable in constant time. Then Algorithm~\ref{alg:pruning} has a run time of at most
    \begin{equation*}
        \O(\delta^{-1} \lambda \log(N / \eta) N^2),
    \end{equation*}
    with probability at least $1 - \eta$.
\end{proposition}

\begin{remark}
    The run time of Algorithm~\ref{alg:pruning} has a bottleneck of $\O(N^2)$ in terms of the number of samples, which may be serious for large data sets. This bottleneck already occurs in the first loop. In Section~\ref{sec:numerical-experiments} we will consider a variation of the algorithm in which the number of hyperplanes drawn in the first layer is a hyperparameter. As we will see in Theorem~\ref{thm:pruning/covering-bound}, this algorithm is guaranteed to succeed with high probability if the number of draws is chosen large enough. In this case, the run time of the algorithm is dictated by the construction of the second layer, which takes time $\O(M^- N^+)$.
\end{remark}

To complement Proposition~\ref{pro:pruning/correctness} we derive a high probability bound on the size of the network produced by Algorithm~\ref{alg:pruning}. This bound will (at least in our proof) depend on the choice of the activation function $\sigma$. We focus on the setting with threshold activations, \ie, we consider
\begin{equation*}
    \sigma(t) = \Thres(t) = \begin{cases}
    1 & \text{if } t > 0,\\
    0 & \text{else}.
    \end{cases}
\end{equation*}

Let us first observe that in the limit, the shape of the activation region of every neuron in the second layer is a Euclidean ball of a `maximal radius', \ie, that touches the closest point in the set $\set{X}^+$. This gives geometric intuition on why the size of the second layer is naturally connected with the mutual covering numbers.

\begin{proposition}[Limit shape of activation regions---threshold activations]\label{pro:limitShapeAct}
    Take\\any $\vec{x}_*^-\in \mathcal{X}^-$ and let $\set{A}_{\vec{x}_*^-}$ be the activation region of $\hat{\varphi}_{\vec{x}_*^-}$. Then, for any $\vec{x} \in \R^d \setminus \partial \set{B}_{\vec{x}_*^-}$,
    \begin{equation*}
        \lim_{\lambda\to\infty}\lim_{n\to \infty} \1_{\set{A}_{\vec{x}_*^-}}(\vec{x})
        = \1_{\set{B}_{\vec{x}_*^-}}(\vec{x})
    \end{equation*}
    almost surely, where $\set{B}_{\vec{x}_*^-} = \B_2^d(\vec{x}_*^-; \dist(\vec{x}_*^-,\set{X}^+))$.
\end{proposition}

Let us give an intuitive sketch for the proof of Proposition~\ref{pro:limitShapeAct}. Roughly speaking, the neuron $\hat{\varphi}_{\vec{x}_*^-}$ activates when the fraction of hyperplanes separating the reference point $\vec{x}_*^-$ and the input $\vec{x}$ is smaller than a threshold value, which is the minimal fraction of hyperplanes separating $\vec{x}_*^-$ and any $\vec{x}^+ \in \set{X}^+$. If $n\to \infty$, then the fraction of hyperplanes separating $\vec{x}_*^-$ and any $\vec{z}$ becomes proportional to the probability that a hyperplane separates the two. Finally, as $\lambda\to\infty$, this probability becomes proportional to $\dist(\vec{x}_*^-,\vec{z})/\lambda$. Hence, in the double limit, the neuron activates when $\dist(\vec{x}_*^-,\vec{x})$ is smaller than $\dist(\vec{x}_*^-,\mathcal{X}^+)$.

Let us now state the main result of our work.

\begin{theorem}[Size of interpolating net---threshold activations]
    \label{thm:pruning/covering-bound}
    Let $\set{X}^-, \set{X}^+ \subset R \B_2^d$ be finite and disjoint. Let $\sigma$ be the threshold activation and $\lambda \gtrsim R$.
    Suppose that there is a mutual covering of $\set{X}^-$ and $\set{X}^+$ such that the centers $\set{C}^-$ and $\set{C}^+$ are $\delta$-separated and the radii satisfy
    \begin{equation*}
        r_{\ell}^-\lesssim \frac{\dist(\vec{c}_{\ell}^-,\mathcal{C}^+)}{\log^{1/2}(e\lambda/\dist(\vec{c}_{\ell}^-,\mathcal{C}^+))}
        \quad\text{and}\quad
        r_{j}^+\lesssim \frac{\dist(\vec{c}_{j}^+,\mathcal{C}^-)}{\log^{1/2}(e\lambda/\dist(\vec{c}_{j}^+,\mathcal{C}^-))}
    \end{equation*}
    for all $\ell \in \range{M^-}$ and $j \in \range{M^+}$. Set $\omega \coloneqq \max\{\omega^-, \omega^+\}$ where
    \begin{equation*}
        \omega^- \coloneqq \max_{\ell \in \range{M^-}} \frac{\gw^2(\set{X}_\ell^- - \vec{c}_\ell^-)}{\dist^3(\vec{c}_\ell^-, \set{C}^+)}
        \quad\text{and}\quad
        \omega^+ \coloneqq \max_{j \in \range{M^+}} \frac{\gw^2(\set{X}_j^+ - \vec{c}_j^+)}{\dist^3(\vec{c}_j^+, \set{C}^-)}.
    \end{equation*}
    Suppose that
    \begin{equation}
        \label{eqn:condnminMain}
        n_{\min} \gtrsim \lambda\delta^{-1} \log(2 M^- M^+ / \eta) + \lambda\omega.
    \end{equation}
    Then, with probability at least $1-\eta$, the neural network computed by Algorithm~\ref{alg:pruning} has layer widths $n=n_{\min}$ and $\hat{n} \leq M^-$.
\end{theorem}

\begin{remark}
\label{rem:omegaEst}
    We give a few examples of estimates of the Gaussian mean width \citep[see, \eg,][for further details]{vershynin2018high} to highlight some special cases of the condition~\eqref{eqn:condnminMain}.
    \begin{enumerate}
        \item For a finite set $\set{A} \subset \B_2^d$ we have $\gw(\set{A}) \lesssim \sqrt{\log(|\set{A}|)}$. As Algorithm~\ref{alg:pruning} requires a finite number $N$ of input samples, $\omega \lesssim \delta^{-1} \log(N)$.
        \item If $\set{A} \subset \B_2^d$ lies in a $k$-dimensional subspace, then $\gw(\set{A}) \lesssim \sqrt{k}$. Hence, for samples in a $k$-dimensional subspace, $\omega \lesssim \delta^{-1} k$.
        \item The set $\Sigma_s^d \coloneqq \{\vec{x} \in \B_2^d : \norm{\vec{x}}_{0} \leq s\}$ of $s$-sparse vectors in the unit ball, where $\norm{\vec{x}}_0$ counts the number of non-zero coordinates in $\vec{x}$, satisfies $\gw(\Sigma_s^d) \lesssim \sqrt{s \log(ed/s)}$. Hence, if the input samples are $s$-sparse, $\omega \lesssim \delta^{-1} s \log(ed/s)$.
    \end{enumerate}
    Notice that the latter two estimates are independent of the number of samples.
\end{remark}

The idea of the proof of Theorem~\ref{thm:pruning/covering-bound} is to show that if $\Phi$ is wide enough, then the neuron $\hat{\varphi}_{\vec{x}_*^-}$ associated with $\vec{x}_*^-$ (defined in \eqref{eqn:dedNeur}) not only separates $\Phi(\vec{x}_*^-)$ and $\Phi(\set{X}^+)$, but in fact acts as a \emph{robust separator}: it will also separate $\Phi(\vec{x}^-)$ and $\Phi(\set{X}^+)$ for all points $\vec{x}^-$ `close enough to' $\vec{x}_*^-$. The key formal observation is stated below in Lemma~\ref{lem:pruning/lemma}. Intuitively, the notion of `close enough' should be relative to the distance of $\vec{x}_*^-$ to the decision boundary. As a result, the size of the interpolating neural net is related to the `complexity' of a mutual covering of $\set{X}^-$ and $\set{X}^+$ in which only the parts of $\set{X}^-$ and $\set{X}^+$ that lie close to the decision boundary need to be covered using components with small diameter---other parts can be crudely covered using large components (see Figure~\ref{fig:mutual-covering}).

Finally, we prove that the statement of Theorem~\ref{thm:pruning/covering-bound} cannot be improved in a certain sense. Proposition~\ref{prop:pruning/sharpness-of-upper-bound} below shows that the upper bound on the size of the second layer $\hat{\Phi}$, as stated in Theorem~\ref{thm:pruning/covering-bound}, cannot be improved in general, assuming that, in addition, $\sigma$ is non-decreasing. Note that this assumption is satisfied by many popular activations, including the ReLU. In the proof, we construct a one-dimensional data set of points with alternating labels, which one could however embed (\eg, by appending zeros) into $\R^d$ for an arbitrary dimension $d \geq 1$. Note that the result holds independently of the random sampling of the first layer, so one cannot even find a benign choice of hyperplanes to improve the situation described below.

\begin{proposition}\label{prop:pruning/sharpness-of-upper-bound}
    Assume that $\sigma$ is non-decreasing, $\sigma(t)=0$ for $t\leq 0$ and $\sigma(t)>0$ for $t>0$. Let $M^- \geq 2$ and $M^+ \coloneqq M^- - 1$. Then, for all $N^- \geq M^-$ and $N^+ \geq M^+$, there exists $\set{X}^-, \set{X}^+ \subset [0, 1]$ with $N^- = |\set{X}^-|$ and $N^+ = |\set{X}^+|$, and a mutual covering $\set{C}^- = \{c_1^-, \dots, c_{M^-}^-\} \subset \set{X}^-$ and $\set{C}^+ = \{c_1^+, \dots, c_{M^+}^+\} \subset \set{X}^+$ such that the output $F$ of Algorithm~\ref{alg:pruning} has at least $M^-$ neurons in its second layer.
\end{proposition}

\section{Proofs}
\label{sec:proofs}

In this section, we present the proofs that have previously been omitted.

\subsection{Proof of Proposition~\ref{pro:pruning/correctness}}
\label{sec:proofCorrectness}

We use the following terminology. 
\begin{definition}\label{def:hyperplane-separation}
    Let $\vec{v} \in \R^d \setminus \{\vec{0}\}$, $\tau \in \R$ and $t \geq 0$. A hyperplane $H[\vec{v}, \tau]$ \emph{$t$-separates $\set{X}^-$ from $\set{X}^+$} if
    \begin{equation*}
        \begin{aligned}
            \langle \vec{v}, \vec{x}^- \rangle + \tau & \leq -t \qquad \text{for all } \vec{x}^- \in \set{X}^-,\\
            \langle \vec{v}, \vec{x}^+ \rangle + \tau & > +t \qquad \text{for all } \vec{x}^+ \in \set{X}^+.
        \end{aligned}
    \end{equation*}
    If $t = 0$, we simply say that $H[\vec{v}, \tau]$ \emph{separates $\set{X}^-$ from $\set{X}^+$.}
\end{definition}

To prove Proposition~\ref{pro:pruning/correctness} it suffices to prove the following statement. It shows that the probability that the first loop of Algorithm~\ref{alg:pruning} stops, and hence the algorithm terminates, increases exponentially in terms of the number of hyperplanes $n$. Allowing $n$ to grow unbounded then directly yields Proposition~\ref{pro:pruning/correctness}.
\begin{proposition}
    \label{pro:interpolation/termination}
    Let $\set{X}^-, \set{X}^+ \subset R\B_2^d$ be finite and $\delta$-separated with $N^- \coloneqq |\set{X}^-|$ and $N^+ \coloneqq |\set{X}^+|$. Let $\lambda \gtrsim R$. Assume that the loop in Algorithm~\ref{alg:pruning} ran for at least $n$ iterations, where
    \begin{equation}\label{eq:interpolation/termination/condition-on-n}
        n \gtrsim \delta^{-1} \lambda \cdot \log(N^- N^+ / \eta).
    \end{equation}
    Then, the exit condition of the loop is satisfied with probability at least $1 - \eta$.
\end{proposition}
In the proof, we will use the following lower bound on the probability that a random hyperplane from Algorithm~\ref{alg:pruning} separates a fixed pair of points.

\begin{lemma}
    \cite[Theorem~18]{dirksen2022separation}
    \label{lem:prob_single_hyperplane_separates_two_points}
    There is an absolute constant $c > 0$ such that the following holds. Let $\vec{x}^-, \vec{x}^+ \in R\B_2^d$. Let $\vec{g} \in \R^d$ denote a standard Gaussian random vector and let $\tau\in [-\lambda,\lambda]$ be uniformly distributed. If $\lambda\gtrsim R$, then with probability at least $c\|\vec{x}^+ - \vec{x}^-\|_2 / \lambda$, the hyperplane $H[\vec{g},\tau]$ $\|\vec{x}^+ - \vec{x}^-\|_2$-separates $\vec{x}^-$ from $\vec{x}^+$.
\end{lemma}

\begin{proof}[Proposition~\ref{pro:interpolation/termination}]
    Fix $\vec{x}^- \in \set{X}^-$ and $\vec{x}^+ \in \set{X}^+$. We consider i.i.d.\ copies $H_1,\ldots,H_n$ of a hyperplane $H = H[\vec{w}, b]$, where $\vec{w} \sim \Normal(\vec{0}, \mat{I}_d)$ and $b \sim \Unif([-\lambda, \lambda])$ are independent. By Lemma~\ref{lem:prob_single_hyperplane_separates_two_points}, the probability that $\vec{x}^-$ and $\vec{x}^+$ is not separated by any of these hyperplanes is at most $(1 - c\delta/\lambda)^n$. By taking a union bound over all $N^- N^+$ pairs of points, we see that the probability that at least one pair has no separating hyperplane is at most 
    \begin{equation*}
        N^- N^+ \left( 1 - c \frac{\delta}{\lambda} \right)^n
        \leq N^- N^+ e^{- c \frac{\delta}{\lambda} n}
        \leq \eta,
    \end{equation*}
    where we used that $1 + x \leq e^x$ for $x \in \R$ and the last inequality follows from~\eqref{eq:interpolation/termination/condition-on-n}.
\end{proof}

\subsection{Proof of Proposition~\ref{pro:run-time}}

The calculation of the radius takes time $\O(N^- + N^+)$. The loops run for $n$ and $\hat{n}$ iterations where each iteration takes time $\O(N^- N^+)$ and $\O(n (N^- + N^+))$, respectively. Transforming all samples once with the first layer (which is needed to compute the second loop) takes time $\O(n (N^- + N^+))$. This totals $\O(n (N^- N^+ + \hat{n} N^- + \hat{n} N^+)) = \O(n N^2)$, where we used that $\hat{n} \leq N$. Applying Proposition~\ref{pro:interpolation/termination} completes the proof.

\subsection{Proof of Proposition~\ref{pro:limitShapeAct}}

Recall that the neuron $\hat{\varphi}_{\vec{x}_*^-}$ activates on $\vec{x}\in \R^d$ if and only if
\begin{equation*}
    \langle \vec{u}_{\vec{x}_*^-}, \Phi(\vec{x}) \rangle < m_{\vec{x}_*^-} = \min_{\vec{x}^+ \in \set{X}^+} \langle \vec{u}_{\vec{x}_*^-}, \Phi(\vec{x}^+) \rangle.
\end{equation*}

We make two observations. First, for any $\vec{x}\in \R^d$,
\begin{equation*}
    \langle\vec{u}_{\vec{x}_*^-}, \Phi(\vec{x}) \rangle = \sum_{i=1}^n \1_{\{\Phi(\vec{x}_*^-)_i=0\}}\Phi(\vec{x})_i = \sum_{i=1}^n \1_{\{\langle \vec{w}_i, \vec{x}_*^-\rangle+b_i\leq 0<\langle \vec{w}_i, \vec{x}\rangle+b_i\}},
\end{equation*}
and hence, by the law of large numbers and by symmetry,
\begin{equation}\label{eq:limitShapeAct/limit-n}
    \lim_{n\to\infty} \frac{1}{n}\langle\vec{u}_{\vec{x}_*^-}, \Phi(\vec{x}) \rangle = \frac{1}{2}\bP(\sign(\langle \vec{w}, \vec{x}_*^-\rangle+b)\neq \sign(\langle \vec{w}, \vec{x}\rangle+b))
\end{equation}
almost surely, where $\vec{w} \sim \Normal(\vec{0}, \mat{I}_d)$ and $b \sim \Unif([-\lambda,\lambda])$ are independent. Second, by \cite[Lemma A.1]{DMS22}, for any $\vec{x},\vec{y}\in \R^d$,
\begin{align*}
    & 2\lambda\bP_{b}(\sign(\langle \vec{w},\vec{x}\rangle+b)\neq \sign(\langle \vec{w},\vec{y}\rangle+b)) \nonumber\\
    & \ \ = |\langle \vec{w},\vec{x}-\vec{y}\rangle| 1_{\{|\langle \vec{w},\vec{x}\rangle|\leq \lambda, |\langle \vec{w},\vec{y}\rangle|\leq \lambda\}}\\
    & \ \ \ \ + 2\lambda (1_{\{\langle \vec{w},\vec{x}\rangle>\lambda, \langle \vec{w},\vec{y}\rangle<-\lambda\}} + 1_{\{\langle \vec{w},\vec{x}\rangle<-\lambda, \langle \vec{w},\vec{y}\rangle>\lambda\}})  \nonumber\\
    & \ \ \ \ + (\lambda-\langle \vec{w},\vec{x}\rangle) 1_{\{\langle \vec{w},\vec{y}\rangle>\lambda,|\langle \vec{w},\vec{x}\rangle|\leq \lambda\}} + (\lambda-\langle \vec{w},\vec{y}\rangle) 1_{\{\langle \vec{w},\vec{x}\rangle>\lambda,|\langle \vec{w},\vec{y}\rangle|\leq \lambda\}} \nonumber \\ 
    & \ \ \ \ + (\lambda+\langle \vec{w},\vec{x}\rangle) 1_{\{\langle \vec{w},\vec{y}\rangle<-\lambda,|\langle \vec{w},\vec{x}\rangle|\leq \lambda\}}  + (\lambda+\langle \vec{w},\vec{y}\rangle) 1_{\{\langle \vec{w},\vec{x}\rangle<-\lambda,|\langle \vec{w},\vec{y}\rangle|\leq \lambda\}}, 
\end{align*}
where $\bP_{b}$ is the probability with respect to $b$. As $\bP(|\langle \vec{w},\vec{z}\rangle|>\lambda)\leq 2e^{-c\lambda^2/\|\vec{z}\|_2^2}$ for any $\vec{z}\in \R^d$, we find by taking expectations with respect to $\vec{w}$, taking the limit for $\lambda\to\infty$, and using monotone convergence that
\begin{equation}\label{eq:limitShapeAct/limit-lamb}
    \lim_{\lambda\to \infty} 2\lambda\bP(\sign(\langle \vec{w},\vec{x}\rangle+b) \neq \sign(\langle \vec{w},\vec{y}\rangle+b))
    = \E|\langle \vec{w},\vec{x}-\vec{y}\rangle|
    =\sqrt{2/\pi}\|\vec{x}-\vec{y}\|_2.
\end{equation}

We proceed with the proof by distinguishing two cases. Let $\vec{x} \in \R^d$, assume $\norm{\vec{x}_*^- - \vec{x}}_2 < \min_{\vec{x}^+ \in \set{X}^+} \norm{\vec{x}_*^- - \vec{x}^+}_2$ and define
\begin{equation*}
    \varepsilon \coloneqq \frac{\min_{\vec{x}^+ \in \set{X}^+} \norm{\vec{x}_*^- - \vec{x}^+}_2 - \norm{\vec{x}_*^- - \vec{x}}_2}{2} > 0.
\end{equation*}
By~\eqref{eq:limitShapeAct/limit-lamb}, there exists $\Lambda > 0$ such that for $\lambda > \Lambda$,
\begin{equation*}
\begin{aligned}
    & \sqrt{2 \pi} \lambda\bP(\sign(\langle \vec{w},\vec{x}_*^-\rangle+b)\neq \sign(\langle \vec{w},\vec{x}\rangle+b))\\
    & \qquad < \norm{\vec{x}_*^- - \vec{x}}_2 + \varepsilon\\
    & \qquad = \min_{\vec{x}^+ \in \set{X}^+} \norm{\vec{x}_*^- - \vec{x}^+}_2 - \varepsilon\\
    & \qquad < \min_{\vec{x}^+ \in \set{X}^+} \sqrt{2 \pi} \lambda\bP(\sign(\langle \vec{w},\vec{x}_*^-\rangle+b)\neq \sign(\langle \vec{w},\vec{x}^+\rangle+b)),
\end{aligned}
\end{equation*}
and hence
\begin{equation*}
\begin{aligned}
    & \bP(\sign(\langle \vec{w},\vec{x}_*^-\rangle+b)\neq \sign(\langle \vec{w},\vec{x}\rangle+b))\\
    & \qquad < \min_{\vec{x}^+ \in \set{X}^+} \bP(\sign(\langle \vec{w},\vec{x}_*^-\rangle+b)\neq \sign(\langle \vec{w},\vec{x}^+\rangle+b)).
\end{aligned}
\end{equation*}
Further, define
\begin{equation*}
\begin{aligned}
    \delta \coloneqq \frac{1}{2} \bigl( & \min_{\vec{x}^+ \in \set{X}^+} \bP(\sign(\langle \vec{w},\vec{x}_*^-\rangle+b)\neq \sign(\langle \vec{w},\vec{x}^+\rangle+b)) \\
    & \qquad \qquad - \bP(\sign(\langle \vec{w},\vec{x}_*^-\rangle+b)\neq \sign(\langle \vec{w},\vec{x}\rangle+b)) \bigr) > 0.
\end{aligned}
\end{equation*}
By~\eqref{eq:limitShapeAct/limit-n}, almost surely, there exists $N \in \N$ such that for $n > N$,
\begin{equation*}
\begin{aligned}
    \frac{2}{n}\langle\vec{u}_{\vec{x}_*^-}, \Phi(\vec{x}) \rangle
    & < \bP(\sign(\langle \vec{w},\vec{x}_*^-\rangle+b)\neq \sign(\langle \vec{w},\vec{x}\rangle+b)) + \delta\\
    & = \min_{\vec{x}^+ \in \set{X}^+} \bP(\sign(\langle \vec{w},\vec{x}_*^-\rangle+b)\neq \sign(\langle \vec{w},\vec{x}^+\rangle+b)) - \delta\\
    & < \min_{\vec{x}^+ \in \set{X}^+} \frac{2}{n} \langle\vec{u}_{\vec{x}_*^-}, \Phi(\vec{x}^+) \rangle,
\end{aligned}
\end{equation*}
and hence
\begin{equation*}
    \langle\vec{u}_{\vec{x}_*^-}, \Phi(\vec{x}) \rangle
    < \min_{\vec{x}^+ \in \set{X}^+} \langle\vec{u}_{\vec{x}_*^-}, \Phi(\vec{x}^+) \rangle.
\end{equation*}
This shows that
\begin{equation*}
    \lim_{\lambda \to \infty} \lim_{n \to \infty} \1_{\set{A}_{\vec{x}_*^-}}(\vec{x}) = \1_{\set{B}_{\vec{x}_*^-}}(\vec{x})
\end{equation*}
almost surely if $\norm{\vec{x}_*^- - \vec{x}}_2 < \min_{\vec{x}^+ \in \set{X}^+} \norm{\vec{x}_*^- - \vec{x}^+}_2$. The remaining case $\norm{\vec{x}_*^- - \vec{x}}_2 > \min_{\vec{x}^+ \in \set{X}^+} \norm{\vec{x}_*^- - \vec{x}^+}_2$ can be proved with only minor changes and is omitted.

\subsection{Proof of Theorem~\ref{thm:pruning/covering-bound}}

The key observation to prove Theorem~\ref{thm:pruning/covering-bound} is stated in Lemma~\ref{lem:pruning/lemma}. To prove it we will need two ingredients. The first is a slight modification of~\cite[Theorem 26]{dirksen2022separation}.

\begin{lemma}
\label{lem:Thres_NN_finite_1st_layer}
    There exists an absolute constant $c>0$ such that the following holds. Let $\set{X}^-, \set{X}^+ \subset R\B_2^d$ be $\delta$-separated sets with $N^- := |\set{X}^-|$,  $N^+ := |\set{X}^+|$. Let $\mat{W} \in \R^{n \times d}$ be a matrix with standard Gaussian entries, $\vec{b} \in \R^n$ be uniformly distributed in $[-\lambda, \lambda]^n$ and let $\mat{W}$ and $b$ be independent. Consider the associated random threshold layer $\Phi\colon \R^d \to \R^{n}$
    \begin{equation*}
        \Phi(\vec{x}) = \Thres(\mat{W} \vec{x} + \vec{b}),
        \quad \vec{x} \in \R^d.
    \end{equation*}	
    Suppose that $\lambda \gtrsim R$ and 
    \begin{equation}\label{eq:Thres_NN_finite_1st_layer/condition_n}
        n \gtrsim \delta^{-1} \lambda \cdot \log(2 N^- N^+ / \eta).
    \end{equation}
    Then with probability at least $1-\eta$, the following event occurs: For every $\vec{x}^- \in \set{X}^-$, the vector $\vec{u}_{\vec{x}^-}\in \{0, 1\}^n$ 
    \begin{equation}\label{eq:lin_sep_random_ReLU_NN_finite_1st_layer:direction_minus}
        (\vec{u}_{\vec{x}^-})_i = \begin{cases}
        1, & (\Phi(\vec{x}^-))_i = 0,\\
        0, & \text{otherwise},
        \end{cases}
    \end{equation}
    satisfies $\langle \vec{u}_{\vec{x}^-},\Phi(\vec{x}^-) \rangle=0$ and
    \begin{equation*}
	\langle\vec{u}_{\vec{x}^-},\Phi(\vec{x}^+)\rangle \geq c \|\vec{x}^+-\vec{x}^-\|_2 \cdot \lambda^{-1} n \qquad \text{for all $\vec{x}^+\in \set{X}^+$.}
    \end{equation*}
\end{lemma}

Geometrically, Lemma~\ref{lem:Thres_NN_finite_1st_layer} states that with high probability the hyperplane $H[\vec{u}_{\vec{x}^-},0]$ linearly separates $\Phi(\vec{x}^-)$ from $\Phi(\set{X}^+)$ and the separation margin increases with both $n$ and the distance between $\vec{x}^-$ and $\set{X}^+$.

\begin{proof}
By \eqref{eq:lin_sep_random_ReLU_NN_finite_1st_layer:direction_minus} it is clear that $\langle \vec{u}_{\vec{x}^-}, \Phi(\vec{x}^-) \rangle=0$. 
    Let $\mat{W} = [\vec{w}_1, \dots, \vec{w}_n]^\top \in \R^{n \times d}$ and $\vec{b} = (b_1, \dots, b_n)^\top \in \R^n$ be the weight matrix and bias vector of $\Phi$, respectively. For $\vec{x}^- \in \set{X}^-$ and $\vec{x}^+ \in \set{X}^+$ define
    \begin{equation*}
    \begin{aligned}
        \set{I}_{\vec{x}^-, \vec{x}^+} & = \{ i \in \range{n} : \langle \vec{w}_i, \vec{x}^- \rangle \leq -b_i < \langle \vec{w}_i, \vec{x}^+ \rangle \},
    \end{aligned}
    \end{equation*}
    and define the events
    \begin{equation*}
        B_{\vec{x}^-, \vec{x}^+}^i = \{ H[\vec{w}_i, b_i] \ \norm{\vec{x}^+ - \vec{x}^-}_2\text{-separates } \vec{x}^- \text{ from } \vec{x}^+ \},
    \end{equation*}
    For $n'(\vec{x}^-, \vec{x}^+) > 0$ to be specified later, set
    \begin{equation*}
    \begin{aligned}
        B_{\vec{x}^-, \vec{x}^+, n'(\vec{x}^-, \vec{x}^+)} & = \Big\{ \sum_{i=1}^n \1_{B_{\vec{x}^-, \vec{x}^+}^i} \geq n'(\vec{x}^-, \vec{x}^+) \Big\},\\
        B & = \bigcap_{(\vec{x}^-, \vec{x}^+) \in \set{X}^- \times \set{X}^+} B_{\vec{x}^-, \vec{x}^+, n'(\vec{x}^-, \vec{x}^+)}.
    \end{aligned}
    \end{equation*}
    On the event $B$, for every $\vec{x}^- \in \set{X}^-$ and $\vec{x}^+ \in \set{X}^+$,
    \begin{equation*}
        \langle \vec{u}_{\vec{x}^-}, \Phi(\vec{x}^+) \rangle
        = \sum_{i \in \set{I}_{\vec{x}^-, \vec{x}^+}} \Thres(\langle \vec{w}_i, \vec{x}^+ \rangle + b_i)
        = |\set{I}_{\vec{x}^-, \vec{x}^+}|
        \geq n'(\vec{x}^-, \vec{x}^+).
    \end{equation*}
    For every $i \in \range{n}$, Lemma~\ref{lem:prob_single_hyperplane_separates_two_points} implies that $\bP(B_{\vec{x}^-, \vec{x}^+}^i) \geq c \norm{\vec{x}^+ - \vec{x}^-}_2 \lambda^{-1}$ for an absolute constant $c > 0$ if $\lambda \gtrsim R$. Therefore, Chernoff's inequality for sums of independent Bernoulli random variables \citep[see, e.g.,][Section 2.3]{vershynin2018high} implies that
    \begin{equation*}
        \bP\left(\sum_{i=1}^n \1_{B_{\vec{x}^-, \vec{x}^+}^i} \leq \frac{c}{2} \lambda^{-1} \norm{\vec{x}^+ - \vec{x}^-}_2 n\right) \leq \exp(- c' \lambda^{-1} \norm{\vec{x}^+ - \vec{x}^-}_2 n),
    \end{equation*}
    where $c' > 0$ is an absolute constant. Setting $n'(\vec{x}^-, \vec{x}^+) = \frac{c}{2} \norm{\vec{x}^+ - \vec{x}^-}_2 \lambda^{-1} n$, we obtain
    \begin{equation*}
        \bP(B_{\vec{x}^-, \vec{x}^+, n'(\vec{x}^-, \vec{x}^+)}^c)
        \leq \exp(- c' \lambda^{-1} \norm{\vec{x}^+ - \vec{x}^-}_2 n)
        \leq \exp(- c' \lambda^{-1} \delta n).
    \end{equation*}
    Hence, by the union bound and \eqref{eq:Thres_NN_finite_1st_layer/condition_n},
    \begin{equation*}
        \bP(B^c)
        \leq N^- N^+ \exp(- c' \lambda^{-1} \delta n)
        \leq \eta.
    \end{equation*}
\end{proof}

Our second proof ingredient is the following lemma. It is an immediate consequence
of \cite[Theorem 2.9]{DiM21}.

\begin{lemma}
    \label{lem:signPert}
    Consider $\vec{c}_1,\ldots,\vec{c}_M\subset \R^d$ and $\set{X}_1,\ldots,\set{X}_M\subset \R^d$ such that $\set{X}_j\subset \B_2^d(\vec{c}_j,r_j) \subset R\B_2^d$ for all $j\in [M]$. Let 
    \begin{equation*}
        r_j \lesssim \frac{r_j'}{\sqrt{\log(e \lambda/r_j')}},
        \qquad r' = \min_{j\in [M]} r'_j.
    \end{equation*}
    Let further $\vec{w}_1, \dots, \vec{w}_n\sim \Normal(\vec{0}, \mat{I}_d)$ and $b_1, \dots, b_n\sim \Unif([-\lambda, \lambda])$ all be independent. If $\lambda \gtrsim R$ and
    \begin{align*}
    n \gtrsim \frac{\lambda}{r'}\log (2 M / \eta) + \max_{j\in [M]} \frac{\lambda}{(r'_j)^3}w^2(\set{X}_j - \vec{c}_j),
    \end{align*} 
    then, with probability at least $1 - \eta$, for all $j\in [M]$ and $\vec{x} \in \set{X}_j$,
    \begin{equation*}
        |\{i\in [n] \ : \ \Thres(\langle \vec{w}_i,\vec{c}_j\rangle + b_i)\neq \Thres(\langle \vec{w}_i,\vec{x}\rangle + b_i)\}| \lesssim \frac{r_j' n}{\lambda}.
    \end{equation*}
\end{lemma}

The following result shows that the `dedicated' neuron $\hat{\varphi}_{\vec{x}_*^-}$ associated with $\vec{x}_*^-$ (defined in \eqref{eqn:dedNeur}) not only separates $\Phi(\vec{x}_*^-)$ and $\Phi(\set{X}^+)$, but in fact acts as a robust separator: it also separates $\Phi(\vec{x}^-)$ and $\Phi(\set{X}^+)$ for all points $\vec{x}^-$ in the component of the mutual covering in which $\vec{x}_*^-$ resides.

\begin{lemma}
\label{lem:pruning/lemma}
    Consider the setting of Theorem~\ref{thm:pruning/covering-bound}. For $\vec{x}_*^- \in \set{X}^-$ we define the associated neuron $\hat{\varphi}_{\vec{x}_*^-} \colon \R^n \to \{ 0, 1 \}$ by
    \begin{equation*}
        \hat{\varphi}_{\vec{x}_*^-}(\vec{z}) = \Thres(-\langle \vec{u}_{\vec{x}_*^-}, \vec{z} \rangle + m_{\vec{x}_*^-}),
    \end{equation*}
    where
    \begin{equation*}
        \vec{u}_{\vec{x}_*^-} = \1[\Phi(\vec{x}_*^-) = \vec{0}]
        \quad\text{and}\quad
        m_{\vec{x}_*^-} = \min_{\vec{x}^+ \in \set{X}^+} \langle \vec{u}_{\vec{x}_*^-}, \Phi(\vec{x}^+) \rangle.
    \end{equation*}
    Then, with probability at least $1 - \eta$, for all $\ell \in \range{M^-}$ and $\vec{x}_*^- \in \set{X}_\ell^-$,
    \begin{equation}\label{eqn:pruningLemmaFirstEqn}
        \hat{\varphi}_{\vec{x}_*^-}(\Phi(\vec{x}^-)) > 0 \quad \text{for all } \vec{x}^- \in \set{X}_\ell^-,
    \end{equation}
    and
    \begin{equation}\label{eqn:pruningLemmaSecond}
        \hat{\varphi}_{\vec{x}_*^-}(\Phi(\vec{x}^+)) = 0 \quad \text{for all } \vec{x}^+ \in \set{X}^+.
    \end{equation}
\end{lemma}

\begin{proof}
    Clearly, the choice of $m_{\vec{x}_*^-}$ ensures that \eqref{eqn:pruningLemmaSecond} holds. It remains to show that, with probability at least $1 - \eta$, 
    \begin{equation*}
        \langle \vec{u}_{\vec{x}_*^-}, \Phi(\vec{x}^-) \rangle < m_{\vec{x}_*^-} = \min_{\vec{x}^+ \in \set{X}^+} \langle \vec{u}_{\vec{x}_*^-}, \Phi(\vec{x}^+) \rangle
    \end{equation*}
    for all $\ell \in \range{M^-}$ and $\vec{x}_*^-,\vec{x}^- \in \set{X}_\ell^-$. Let $A$ be the event where, for every $\ell \in \range{M^-}$ and $j \in \range{M^+}$,
    \begin{equation*}
        \langle \vec{u}_{\vec{c}_{\ell}^-}, \Phi(\vec{c}_j^+) \rangle
        \geq c_1 \lambda^{-1} \norm{\vec{c}_{\ell}^- - \vec{c}_j^+}_2 n.
    \end{equation*}
    By Lemma~\ref{lem:Thres_NN_finite_1st_layer}, $\bP(A) \geq 1 - \eta$ under our assumptions. Let $B$ be the event where, for all $\ell\in [M^-]$ and $\vec{x}^{-}\in \set{X}_\ell^{-}$,
    \begin{equation*}
    \begin{aligned}
        \norm{\Phi(\vec{c}_\ell^-) - \Phi(\vec{x}^-)}_1
        & = |\{i\in [n] : \Thres(\langle \vec{w}_i,\vec{c}_\ell^{-}\rangle + b_i)\neq \Thres(\langle \vec{w}_i,\vec{x}^{-}\rangle + b_i)\}|\\
        & \leq c_2 \frac{(r_\ell')^{-} n}{\lambda},
    \end{aligned}
    \end{equation*}
    and, for all $j\in [M^+]$ and $\vec{x}^{+}\in \set{X}_j^{+}$,
    \begin{equation*}
    \begin{aligned}
        \norm{\Phi(\vec{c}_j^+) - \Phi(\vec{x}^+)}_1
        & = |\{i\in [n] : \Thres(\langle \vec{w}_i,\vec{c}_j^{+}\rangle + b_i)\neq \Thres(\langle \vec{w}_i,\vec{x}^{+}\rangle + b_i)\}|\\
        & \leq c_2 \frac{(r_j')^{+} n}{\lambda},
    \end{aligned}
    \end{equation*}
    where 
    \begin{equation*}
        \qquad (r_\ell')^{-}=\frac{c_1}{12 c_2} \dist(\vec{c}_{\ell}^-,\mathcal{C}^+),
        \qquad (r_j')^{+}=\frac{c_1}{4 c_2} \dist(\vec{c}_{j}^+,\mathcal{C}^-).
    \end{equation*}
    By Lemma~\ref{lem:signPert}, $\bP(B) \geq 1 - \eta$ under the stated assumptions. For the remainder of the proof, we condition on the event $A \cap B$.
    
    By using $B$, we find
    \begin{equation*}
    \begin{aligned}
        |\langle \vec{u}_{\vec{x}_*^-}, \Phi(\vec{x}^-) \rangle|
        & = |\langle \vec{u}_{\vec{x}_*^-}, \Phi(\vec{x}^-) - \Phi(\vec{x}_*^-) \rangle|\\
        & \leq \norm{\Phi(\vec{x}^-) - \Phi(\vec{x}_*^-)}_1\\
        & \leq \norm{\Phi(\vec{x}^-) - \Phi(\vec{c}_\ell^-)}_1 + \norm{\Phi(\vec{c}_\ell^-) - \Phi(\vec{x}_*^-)}_1\\
        & \leq 2 c_2 \frac{(r_\ell')^-}{\lambda} n.
    \end{aligned}
    \end{equation*}
    Now pick $j \in \range{M^+}$ and $\vec{x}^+ \in \set{X}_j^+$. Using $A$ and $B$,
    \begin{align*}
        & \langle \vec{u}_{\vec{x}_*^-}, \Phi(\vec{x}^+) \rangle \\
        & \qquad = \langle \vec{u}_{\vec{c}_\ell^-}, \Phi(\vec{c}_j^+) \rangle + \langle \vec{u}_{\vec{x}_*^-} - \vec{u}_{\vec{c}_\ell^-}, \Phi(\vec{c}_j^+) \rangle + \langle \vec{u}_{\vec{x}_*^-}, \Phi(\vec{x}^+) - \Phi(\vec{c}_j^+) \rangle\\
        & \qquad \geq \langle \vec{u}_{\vec{c}_\ell^-}, \Phi(\vec{c}_j^+) \rangle - |\langle \Phi(\vec{x}_*^-) - \Phi(\vec{c}_\ell^-), \Phi(\vec{c}_j^+)\rangle| - |\langle \vec{u}_{\vec{x}_*^-}, \Phi(\vec{x}^+) - \Phi(\vec{c}_j^+) \rangle|\\
        & \qquad \geq \langle \vec{u}_{\vec{c}_\ell^-}, \Phi(\vec{c}_j^+) \rangle - \norm{\Phi(\vec{x}_*^-) - \Phi(\vec{c}_\ell^-)}_1 - \norm{\Phi(\vec{x}^+) - \Phi(\vec{c}_j^+)}_1\\
        & \qquad \geq c_1 \lambda^{-1} \norm{\vec{c}_\ell^- - \vec{c}_j^+}_2 n - c_2 \frac{(r_\ell')^-}{\lambda} n - c_2 \frac{(r_j')^+}{\lambda} n,
    \end{align*}
    where in the second step we used that $\vec{u}_{\vec{x}} = \vec{1} - \Phi(\vec{x})$ due to the threshold activation.
    
    Combining the above we see that, for all $\ell \in \range{M^-}$, $\vec{x}_*^-, \vec{x}^- \in \set{X}_\ell^-$, $j \in \range{M^+}$, and $\vec{x}^+ \in \set{X}_j^+$,
    \begin{equation*}
        \langle \vec{u}_{\vec{x}_*^-}, \Phi(\vec{x}^-) \rangle < \langle \vec{u}_{\vec{x}_*^-}, \Phi(\vec{x}^+) \rangle,
    \end{equation*}
    where we have used that
    \begin{equation*}
        (r_\ell')^{-} < \frac{c_1}{6 c_2} \norm{\vec{c}_\ell^- - \vec{c}_j^+}_2,
        \qquad (r_j')^{+} < \frac{c_1}{2 c_2} \norm{\vec{c}_\ell^- - \vec{c}_j^+}_2.
    \end{equation*}
    Since for any $\vec{x}^+ \in \set{X}^+$ there is some $j \in \range{M^+}$ such that $\vec{x}^+ \in \set{X}_j^+$, we find for all $\ell \in \range{M^-}$ and $\vec{x}_*^-, \vec{x}^- \in \set{X}_\ell^-$,
    \begin{equation*}
        \langle \vec{u}_{\vec{x}_*^-}, \Phi(\vec{x}^-) \rangle
        < \min_{\vec{x}^+ \in \set{X}^+} \langle \vec{u}_{\vec{x}_*^-}, \Phi(\vec{x}^+) \rangle
        = m_{\vec{x}_*^-},
    \end{equation*}
    as desired.
\end{proof}

We can now complete the proof. 
\begin{proof}[Theorem~\ref{thm:pruning/covering-bound}]
    Throughout, we condition on the event from Lemma~\ref{lem:pruning/lemma}. Let us first observe that the first loop of Algorithm~\ref{alg:pruning} terminates after $n_{\min}$ iterations and hence the first layer $\Phi$ of $F$ has width $n_{\min}$. Indeed, taking $\vec{x}_*^-=\vec{x}^-$ in \eqref{eqn:pruningLemmaFirstEqn}, we see that $\hat{\varphi}_{\vec{x}^-}(\vec{x}^-) > 0$ for any $\vec{x}^- \in \set{X}^-$ and hence
    \begin{equation*}
        0
        = \langle \vec{u}_{\vec{x}^-}, \Phi(\vec{x}^-) \rangle
        < m_{\vec{x}^-}
        = \min_{\vec{x}^+ \in \set{X}^+} \langle \vec{u}_{\vec{x}^-}, \Phi(\vec{x}^+) \rangle.
    \end{equation*}
    This estimate implies that for all $\vec{x}^+ \in \set{X}^+$, there must be a hyperplane that separates $\vec{x}^-$ from $\vec{x}^+$.

    Next, using induction we show that the second loop terminates after at most $M^-$ steps, thus $\hat{n} \leq M^-$. In the first iteration, we select $\vec{x}_1^- \in \set{C} = \set{X}^-$ which is part of at least one component of the mutual covering, say $\set{X}_{i_1}^-$. By Lemma~\ref{lem:pruning/lemma}, the associated neuron $\hat{\varphi}_{\vec{x}_1^-}$ activates on all of $\set{X}_{i_1}^-$ and hence $\set{C} \cap \set{X}_{i_1}^- = \emptyset$ after the update. Suppose that the $p$-th iteration finished, thus $\set{C} \cap \set{X}_{i_j}^- = \emptyset$ for all $j \in \range{p}$. We select $\vec{x}_{p+1}^- \in \set{C} \subset \set{X}^- \setminus (\set{X}_{i_1}^- \cup \dots \cup \set{X}_{i_p}^-)$ which must be part of a new component, say $\set{X}_{i_{p+1}}^-$. Again, by the lemma the associated neuron activates on all of the component, and thus, after the update $\set{C} \cap \set{X}_{i_j}^- = \emptyset$ for all $j \in \range{p+1}$. By induction, after at most $M^-$ iterations $\set{C} = \emptyset$ and hence the algorithm terminates with $\hat{n} \leq M^-$.
\end{proof}

\subsection{Proof of Corollary~\ref{cor:support-cover-informal}}

Let $\set{C}^-$ and $\set{C}^+$ denote the centers of the mutual covering. We apply Algorithm~\ref{alg:pruning} to $\set{C}^-$ and $\set{C}^+$ (with $\lambda \approx R$ and $n_\text{min}$ from Theorem~\ref{thm:pruning/covering-bound}) with the following change: when computing the biases in the second layer, instead of taking the minimum only over $\set{C}^+$ we set, for all $\ell \in \range{M^-}$,
\begin{equation*}
    m_{\vec{c}_\ell^-} = \min_{\vec{x}^+ \in \set{X}^+} \langle \vec{u}_{\vec{c}_\ell^-}, \Phi(\vec{x}^+) \rangle.
\end{equation*}
Inspecting the proof of Theorem~\ref{thm:pruning/covering-bound}, we see that with positive probability this network has the asserted size and, moreover, interpolates $\set{X}^-$ and $\set{X}^+$.

\subsection{Proof of Proposition~\ref{prop:pruning/sharpness-of-upper-bound}}

Before we construct a data set that satisfies the properties of the proposition, we make some preliminary observations. Consider any $\set{X}^-, \set{X}^+ \subset \R^d$. Let $\Phi$ denote the first layer of the output $F$ of Algorithm~\ref{alg:pruning}. For a given $\vec{x}_*^- \in \set{X}^-$, consider its associated neuron $\varphi_{\vec{x}_*^-}$ defined in \eqref{eqn:dedNeur}. Consider $\vec{x}_*^- \neq \vec{x}^- \in \set{X}^-$ and for $t \geq 0$ set
\begin{equation*}
    \vec{x}_t = \vec{x}_*^- + t (\vec{x}^- - \vec{x}_*^-),
\end{equation*}
so that $\{\vec{x}_t \ : \ t\geq0\}$ is the ray originating from $\vec{x}_*^-$ and passing through $\vec{x}^-$.

First, we claim that $t \mapsto \langle \vec{u}_{\vec{x}_*^-}, \Phi(\vec{x}_t) \rangle$ is non-decreasing. This is an immediate consequence of the fact that $\Phi_i(\vec{x}_t) \leq \Phi_i(\vec{x}_s)$ for all $i \in \range{n}$ such that $\Phi_i(\vec{x}_*^-) = 0$ and for all $0 \leq t \leq s$. This is clear in the case $\Phi_i(\vec{x}_t) = 0$. Assuming $\Phi_i(\vec{x}_t) > 0$ (and hence $t > 0$), the assumptions imposed on $\sigma$ imply that
\begin{equation*}
    0 < \langle \vec{w}_i, \vec{x}_t \rangle + b_i
    = \langle \vec{w}_i, \vec{x}_*^- \rangle + b_i + t \langle \vec{w}_i, \vec{x}^- - \vec{x}_*^- \rangle.
\end{equation*}
As $t > 0$ and $\langle \vec{w}_i, \vec{x}_*^- \rangle + b_i \leq 0$ due to $\Phi_i(\vec{x}_*^-) = 0$, it follows that
\begin{equation*}
    \langle \vec{w}_i, \vec{x}^- - \vec{x}_*^- \rangle
    > 0.
\end{equation*}
Finally, since $\sigma$ is non-decreasing,
\begin{equation*}
\begin{aligned}
    \Phi_i(\vec{x}_t)
    & = \sigma(\langle \vec{w}_i, \vec{x}_t \rangle + b_i)\\
    & \leq \sigma(\langle \vec{w}_i, \vec{x}_t \rangle + b_i + (s - t) \langle \vec{w}_i, \vec{x}^- - \vec{x}_*^- \rangle)\\
    & = \sigma(\langle \vec{w}_i, \vec{x}_s \rangle + b_i) = \Phi_i(\vec{x}_s),
\end{aligned}
\end{equation*}
proving our claim.

Now let us make the following observation: suppose there is $\vec{x}^+ \in \set{X}^+$ which lies between $\vec{x}_*^-$ and $\vec{x}^-$ in the sense that there exists $t^+ \in (0, 1)$ such that $\vec{x}_{t^+} = \vec{x}^+$. Then, the neuron $\hat{\varphi}_{\vec{x}_*^-}$ does not activate on $\Phi(\vec{x}^-)$. To see this, we simply invoke the above claim, which yields
\begin{equation*}
    m_{\vec{x}_*^-}
    \leq \langle \vec{u}_{\vec{x}_*^-}, \Phi(\vec{x}^+) \rangle
    \leq \langle \vec{u}_{\vec{x}_*^-}, \Phi(\vec{x}^-) \rangle,
\end{equation*}
and directly implies $\hat{\varphi}_{\vec{x}_*^-}(\Phi(\vec{x}^-)) = 0$.

With these observations, we can now prove the statement of the proposition. Consider the interval $[0, 1]$ and place points $\vec{c}_\ell^-$ and $\vec{c}_j^+$ in an alternating fashion on an equispaced grid: formally, for $\ell \in \range{M^-}$ and $j \in \range{M^+}$ we set
\begin{equation*}
    \vec{c}_\ell^- = \frac{\ell - 1}{M^- - 1}
    \quad\text{and}\quad
    \vec{c}_j^+ = \frac{j - 1/2}{M^- - 1}.
\end{equation*}
Let $r_\ell^-$ and $r_j^+$ be as in Theorem~\ref{thm:pruning/covering-bound}. Choose the remaining $N^- - M^-$ points $\vec{x}^- \in \set{X}^-$ and $N^+ - M^+$ points $\vec{x}^+ \in \set{X}^+$ such that for each of them there exists $\ell \in \range{M^-}$ with $\norm{\vec{x}^- - \vec{c}_\ell^-}_2 \leq r_\ell^-$ and $j \in \range{M^+}$ with $\norm{\vec{x}^+ - \vec{c}_j^+}_2 \leq r_j^+$, respectively. Then, $\set{C}^- = \{\vec{c}_1^-, \ldots, \vec{c}_{M^-}^-\}$ and $\set{C}^+ = \{\vec{c}_1^+, \ldots, \vec{c}_{M^+}^+\}$ form a mutual covering of $\set{X}^-$ and $\set{X}^+$ as required by Theorem~\ref{thm:pruning/covering-bound}.

Let $\ell \in \range{M^-}$ be fixed. By our earlier observation, for each $\vec{x}^- \in \set{X}^- \setminus \set{X}_\ell^-$, $\varphi_{\vec{x}^-}(\Phi(\vec{c}_\ell^-)) = 0$, as there is a point $\vec{c}_j^+ \in \set{X}^+$ between $\vec{x}^-$ and $\vec{c}_\ell^-$. Thus, to classify $\vec{c}_\ell^-$ correctly, we need to choose (at least) one neuron corresponding to a point in $\set{X}_\ell^-$. As we need to classify the points $\vec{c}_\ell^-$ for all $\ell \in \range{M^-}$ correctly, we cannot include less than $M^-$ neurons in the second layer.

\section{Numerical Experiments}
\label{sec:numerical-experiments}

In this section, we study the performance of Algorithm~\ref{alg:pruning} through numerical simulations on different data sets.\footnote{Code is available at \url{https://github.com/patrickfinke/memo}. We use Python 3, Scikit-learn, and NumPy.} In particular, we want to investigate how the interpolation probability (approximated as the fraction of a fixed amount of runs that produce an interpolating network) and the width of the second layer respond to changes in the width of the first layer $n$ and the maximal bias $\lambda$. Recall that the algorithm was designed in such a way that it adapts the width of the first layer to guarantee interpolation on the input data. To have free control over this parameter we adapt the algorithm slightly for the experiments.

Hence, we formulate Algorithm~\ref{alg:numerical-experiments} which has both $n$ and $\lambda$ as hyperparameters. As the first layer might be such that not every pair of samples with different labels is separated by at least one hyperplane, we have to adjust the construction of the second layer. We keep track of the set $\set{C}$ of candidate samples whose associated neurons might be accepted into the second layer, the set $\set{U}$ of samples that have yet to be correctly classified by a neuron (the universe), and the set $\set{A}$ of samples whose associated neurons have been accepted into the second layer. Note that $\set{C} \subset \set{U}$ but there might not be equality. The algorithm stops if we either run out of candidates or all points are classified correctly. In every iteration, we draw a candidate sample at random and compute the associated neuron. If the neuron at least correctly classifies the candidate itself, we accept it into the second layer and remove every point that the neuron classifies correctly from both $\set{C}$ and $\set{U}$. This check could be omitted in Algorithm~\ref{alg:pruning} due to the construction of the first layer which also guaranteed that $\set{C} = \set{U}$.

\begin{algorithm}
\caption{Interpolation (experiments)}\label{alg:numerical-experiments}
\begin{algorithmic}[1]
\Require Disjoint and finite $\set{X}^-, \set{X}^+ \subset \R^d$ with $N^- \coloneqq |\set{X}^-|$, $N^+ \coloneqq |\set{X}^+|$, activation $\sigma\colon \R \to \R$ satisfying $\sigma(t) = 0$ for $t \leq 0$ and $\sigma(t) > 0$ for $t > 0$, width of first layer $n \geq 1$, maximal bias $\lambda \geq 0$.
\Ensure A three-layer fully-connected neural network $F\colon \R^d \to \{\pm 1\}$.
\vspace{.5em}
\phase{First layer $\Phi$}
\State Randomly sample $\mat{W} \in \R^{n \times d}$ and $\vec{b} \in \R^n$ where
\begin{equation*}
    \vec{W}_i \sim \Normal(\vec{0}, \mat{I}_d)
    \quad\text{and}\quad
    b_i \sim \Unif([-\lambda, \lambda]).
\end{equation*}
are all independent and define the first layer $\Phi(\vec{x}) = \sigma(\mat{W} \vec{x} + \vec{b})$.
\phase{Second layer $\hat{\Phi}$}
\State Initialize $\set{C} \gets \set{X}^-$, $\set{U} \gets \set{X}^-$ and $\set{A} \gets \emptyset$.
\While{$\set{C} \not= \emptyset$ \textbf{and} $\set{U} \not= \emptyset$}
    \State Select a candidate $\vec{x}_*^- \in \set{C}$ at random and update $\set{C} \gets \set{C} \setminus \{ \vec{x}_*^- \}$.
    \State Calculate $\vec{u}_{\vec{x}_*^-} \in \{0,1\}^n$ and $m_{\vec{x}_*^-} \geq 0$ according to
    \begin{equation*}
        \vec{u}_{\vec{x}_*^-} \gets \1[\Phi(\vec{x}_*^-) = \vec{0}]
        \quad\text{and}\quad
        m_{\vec{x}_*^-} \gets \min_{\vec{x}^+ \in \set{X}^+} \langle \vec{u}_{\vec{x}_*^-}, \Phi(\vec{x}^+) \rangle.
    \end{equation*}
    \If{$m_{\vec{x}_*^-} > 0$}
        \State Calculate $\set{T} \gets \{ \vec{x}^- \in \set{U} : \langle \vec{u}_{\vec{x}_*^-}, \Phi(\vec{x}^-) \rangle < m_{\vec{x}_*^-} \}$.
        \State Update $\set{C}$, $\set{U}$ and $\set{A}$ according to
        \begin{equation*}
            \set{C} \gets \set{C} \setminus \set{T},\quad
            \set{U} \gets \set{U} \setminus \set{T}\quad\text{and}\quad
            \set{A} \gets \set{A} \cup \{ \vec{x}_*^- \}.
        \end{equation*}
    \EndIf
\EndWhile
\State Define $\hat{\Phi}(\vec{z}) = \sigma(-\mat{U} \vec{z} + \vec{m})$ with $\mat{U} \in \R^{|\set{A}| \times n}$ and $\vec{m} \in \R^{|\set{A}|}$ where
\begin{equation*}
    \mat{U} \gets \begin{bmatrix} \vec{u}_{\vec{x}_*^-}^\top \end{bmatrix}_{\vec{x}_*^- \in \set{A}}
    \quad\text{and}\quad
    \vec{m} \gets \begin{bmatrix} m_{\vec{x}_*^-} \end{bmatrix}_{\vec{x}_*^- \in \set{A}}.
\end{equation*}
\phase{Output network $F$}
\State Return $F(\vec{x}) = \sign(-\langle \vec{1}, \hat{\Phi}(\Phi(\vec{x})) \rangle)$. 
\end{algorithmic}
\end{algorithm}

In the following, we present five experiments. First, we focus on the verification of our theoretical results through illustrative experiments on simple data sets. In Section~\ref{sec:experiments/binary-classification}, we apply Algorithm~\ref{alg:numerical-experiments} to the Two Moons data set, which allows us to verify our main result and illustrate the underlying geometric intuition. In Section~\ref{sec:experiments/sample-size-limit} we verify that, in a controlled setting which is guaranteed to satisfy our assumptions, the network size indeed does not depend on the number of samples. Next, we examine the performance on real world data. In Section~\ref{sec:experiments/binaryMNIST} we investigate binary classification subproblems of the MNIST data set. We introduce an extension to multi-class classification in Section~\ref{sec:experiments/multiclass-classification} and apply it to MNIST. Additionally, in Section~\ref{sec:experiments/cifar10}, we consider the CIFAR-10 data set. Finally, we present a worst-case example in Section~\ref{sec:experiments/worst-case}. In all experiments, we let $\sigma$ be the threshold activation.

\subsection{Binary Classification on Two Moons}
\label{sec:experiments/binary-classification}

In this section, we apply Algorithm~\ref{alg:numerical-experiments} to the 2D Two Moons\footnote{See \url{https://scikit-learn.org/stable/modules/generated/sklearn.datasets.make_moons.html}.} data set (Figure~\ref{fig:experiments/moons/dataset}), allowing us to easily visualize the output of the algorithm in the input domain. While this is only a synthetic toy data set, it provides a clear geometric structure with well-separated classes. At the same time, the data is not linearly separable, and not all pairs of samples with different labels can be efficiently separated by hyperplanes that pass through the origin, making it a good first testing ground for the effect of the parameter $\lambda$.

\paragraph{Interpolation probability.}

In Figure~\ref{fig:experiments/moons/interpolation} we observe a clear phase transition in the interpolation probability which is in line with the prediction of Theorem~\ref{thm:pruning/covering-bound}, where we treat all complexity terms depending on the data set as constant. As can be seen from the contour lines, for $\lambda$ larger than the data radius, $n \gtrsim \lambda$ is enough to guarantee interpolation with any fixed probability. On the other hand, one can observe that a large enough $\lambda$ is also necessary for efficient interpolation, as for $\lambda = 0$ interpolation does not happen for any value of $n$.

It is noteworthy that the optimal value of $\lambda$ is smaller than the data radius. This is intuitive here, as a maximal bias exceeding the radius of $\set{X}^+$ already guarantees the efficient separation of pairs of opposite labels in the first layer.

\paragraph{Width of the second layer $\hat{\Phi}$.}

As can be seen in Figure~\ref{fig:experiments/moons/pruning}, the width of the second layer becomes much smaller than the number of points. We are mainly interested in the part of the parameter space where the interpolation probability is close to one. In this region, the width attains its minimum and is essentially constant.

\bigskip

Due to the two-dimensionality of the data, it is possible to visualize the decision boundary of our method in input space, see Figure~\ref{fig:experiments/moons/decision-boundary-with-pruning}. Neurons of the second layer have (approximately) circular activation regions that are centered at their corresponding candidate points and which extend all the way to the other class. The third layer takes a union of these regions---the boundary of this union is the decision boundary. We can repeat this visualization for different values of the hyperparameters, see Figure~\ref{fig:experiments/moons/decision-boundary-grid}. For $\lambda = 0$ the method fails to separate pairs of samples with opposite labels because all hyperplanes pass through the origin. If $\lambda$ is large enough and as $n$ grows, the method begins to succeed. In line with Proposition~\ref{pro:limitShapeAct}, the activation regions of the individual neurons become more circular as $n$ increases, which can be best seen in the rightmost column of Figure~\ref{fig:experiments/moons/decision-boundary-grid}.

\begin{figure}
    \centering
    \begin{subfigure}{\textwidth}
        \centering
        \includegraphics{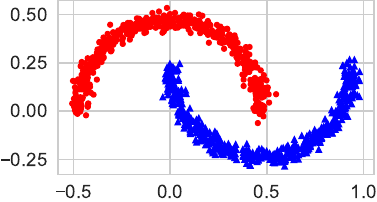}
        \caption{\textbf{Two Moons.} A $d = 2$ dimensional data set of two interleaving half circles. Each class has $N^- = N^+ = 500$ samples and the radius is $R = 1$.}
        \label{fig:experiments/moons/dataset}
    \end{subfigure}

    \begin{subfigure}{\textwidth}
        \centering
        \parbox{0.95\textwidth}{
            \parbox{.45\textwidth}{%
                \includegraphics{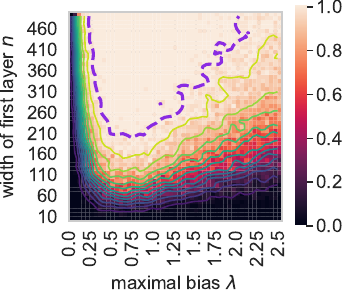}   
            }
            \parbox{.45\textwidth}{%
                \includegraphics{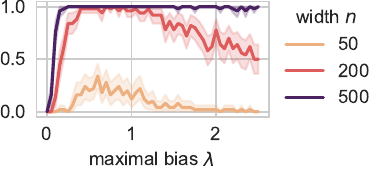}
                \includegraphics{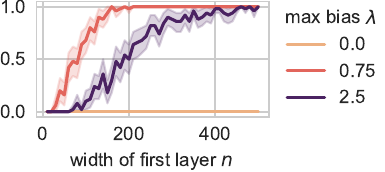}
            }
        }
        \caption{\textbf{Interpolation probability.} (Left) The interpolation probability (average over 250 runs) as a function of the width of the first layer $n$ and maximal bias $\lambda$. The 99\% contour line is at the dashed purple line. (Right) Horizontal (top) and vertical (bottom) slices of the heatmap with $95\%$ confidence intervals.}
        \label{fig:experiments/moons/interpolation}
    \end{subfigure}

    \begin{subfigure}{\textwidth}
        \centering
        \parbox{0.95\textwidth}{
            \parbox{.45\textwidth}{%
                \includegraphics{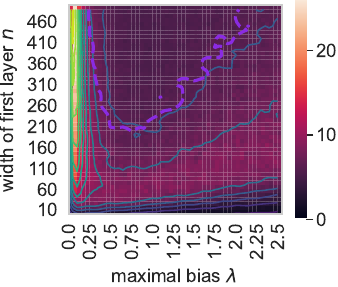}
            }
            \parbox{.45\textwidth}{%
                \includegraphics{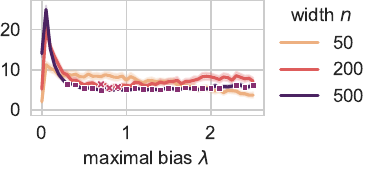}
                \includegraphics{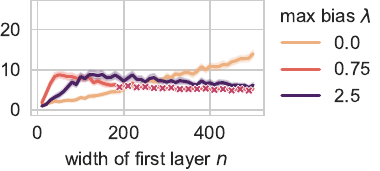}
            }
        }
        \caption{\textbf{Width of the second layer $\hat{\Phi}$.} (Left) The width of the second layer $\hat{\Phi}$ (average over 250 runs) as a function of the width of the first layer $n$ and maximal bias $\lambda$. The 99\% interpolation probability contour line is at the dashed purple line. (Right) Horizontal (top) and vertical (bottom) slices of the heatmap with $95\%$ confidence intervals. Markers indicate an interpolation probability $\geq 99\%$, compare Figure~\ref{fig:experiments/moons/interpolation}.}
        \label{fig:experiments/moons/pruning}
    \end{subfigure}

    \caption{Binary classification on the Two Moons data set.}
    \label{fig:experiments/moons}
\end{figure}

\begin{figure}
    \centering
    \includegraphics{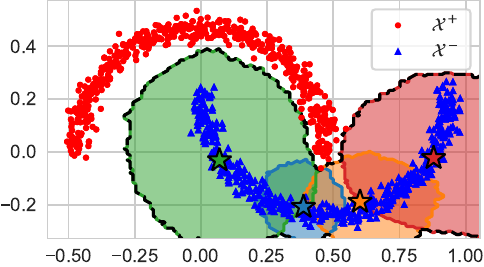}
    \caption{\textbf{Decision boundary.} Each star marks an accepted point and the region of the same color is the activation region of its associated neuron. The decision boundary of the network is the boundary of the union of these regions. Here, we used $n = 2\,000$ and $\lambda = 1$.}
    \label{fig:experiments/moons/decision-boundary-with-pruning}
\end{figure}

\begin{figure}
    \centering
    \includegraphics[width=.32\linewidth]{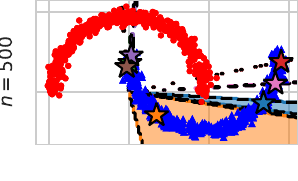}\hspace{-0.0ex}%
    \includegraphics[width=.32\linewidth]{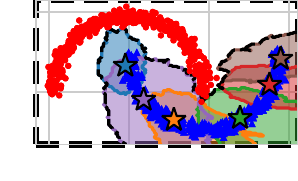}\hspace{-0.0ex}%
    \includegraphics[width=.32\linewidth]{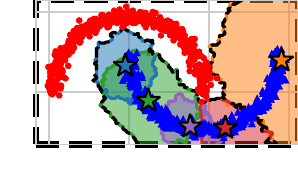}\vspace{-1ex}
    \includegraphics[width=.32\linewidth]{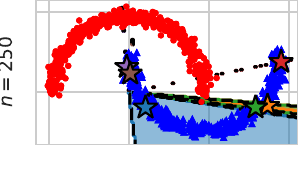}\hspace{-0.0ex}%
    \includegraphics[width=.32\linewidth]{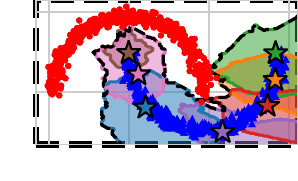}\hspace{-0.0ex}%
    \includegraphics[width=.32\linewidth]{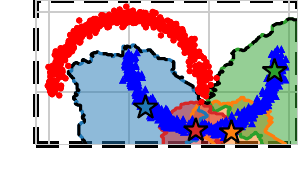}\vspace{-1ex}
    \includegraphics[width=.32\linewidth]{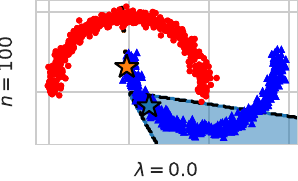}\hspace{-0.0ex}%
    \includegraphics[width=.32\linewidth]{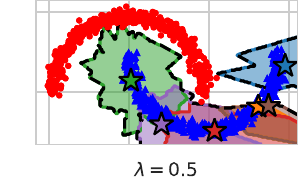}\hspace{-0.0ex}%
    \includegraphics[width=.32\linewidth]{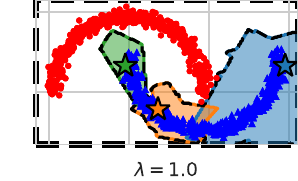}
    \caption{\textbf{Decision boundaries for different choices of hyper-parameters.} Similar to Figure~\ref{fig:experiments/moons/decision-boundary-with-pruning} but includes all combinations of hyper-parameters $n \in  \{100, 250, 500\}$ (rows) and $\lambda \in \{0, 0.5, 1\}$ (columns). Plots in which the network interpolates the data are marked with a thick dashed frame.}
    \label{fig:experiments/moons/decision-boundary-grid}
\end{figure}

\subsection{Behaviour in the Sample Size Limit}
\label{sec:experiments/sample-size-limit}

In Theorem~\ref{thm:pruning/covering-bound}, the size of the interpolating network is independent of the number of samples and only dictated by the parameters of the mutual covering. To illustrate this numerically, we consider a scenario where we sample points from a distribution whose support consists of two disjoint, compact sets representing two classes. We expect that as we iteratively sample points from the distribution, the size of the interpolating network should saturate and be bounded by the parameters of the mutual covering of the support of the distribution (satisfying the restrictions in Theorem~\ref{thm:pruning/covering-bound}).

To verify this, we return to the Two Moons data set from Section~\ref{sec:experiments/binary-classification}. We fix the maximal bias $\lambda = 1$ and vary the number of points $N$ by drawing samples from the data distribution.\footnote{We use \code{sklearn.datasets.make\_moons(n\_samples=N, noise=0.05)} from the scikit-learn Python package to generate the samples.}

\begin{figure}
    \centering
    \begin{subfigure}{\textwidth}
        \centering
        \includegraphics[width=.3\linewidth]{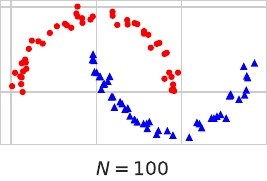}
        \includegraphics[width=.3\linewidth]{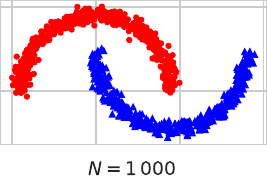}
        \includegraphics[width=.3\linewidth]{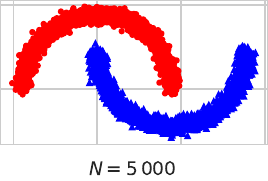}
        \caption{\textbf{Two Moons.} Continuously drawing samples from the distribution of Two Moons (Figure~\ref{fig:experiments/moons/dataset}) leads to a better representation of the support. Each class has always $N/2$ samples and the radius is $R = 1$.}
        \label{fig:experiments/moons-size/dataset}
    \end{subfigure}

    \begin{subfigure}{\textwidth}
        \centering
        \parbox{.95\textwidth}{
            \parbox{.45\textwidth}{%
                \includegraphics{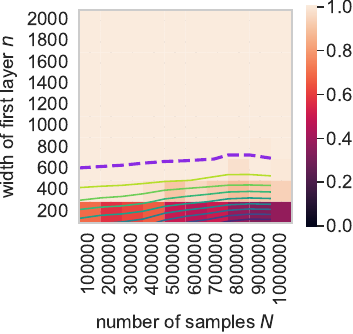}
            }
            \parbox{.45\textwidth}{%
                \includegraphics{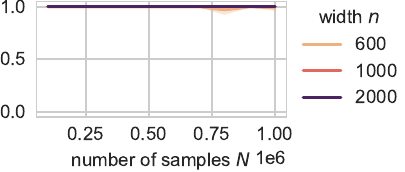}
                \includegraphics{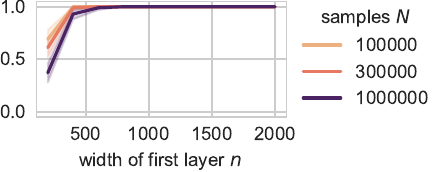}
            }
        }
        \caption{\textbf{Interpolation probability.} (Left) The interpolation probability (average over 100 runs) as a function of the width of the first layer $n$ and the number of samples $N$. The 99\% contour line is at the dashed purple line. (Right) Horizontal (top) and vertical (bottom) slices of the heatmap with $95\%$ confidence intervals.}
        \label{fig:experiments/moons-size/interpolation}
    \end{subfigure}

    \begin{subfigure}{\textwidth}
        \centering
        \parbox{0.95\textwidth}{
            \parbox{.45\textwidth}{%
                \includegraphics{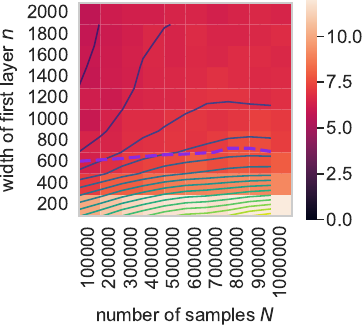}
            }
            \parbox{.45\textwidth}{%
                \includegraphics{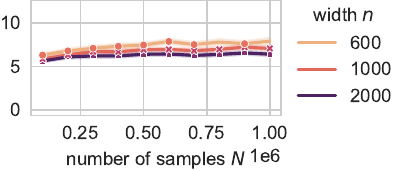}
                \includegraphics{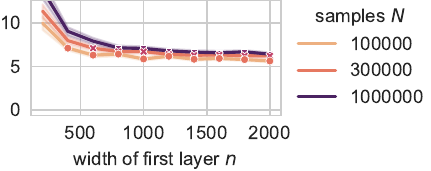}
            }
        }
        \caption{\textbf{Width of the second layer $\hat{\Phi}$.} (Left) The width of the second layer $\hat{\Phi}$ (average over 100 runs) as a function of the width of the first layer $n$ and the number of samples $N$. The 99\% interpolation probability contour line is at the dashed purple line. (Right) Horizontal (top) and vertical (bottom) slices of the heatmap with $95\%$ confidence intervals. Markers indicate an interpolation probability $\geq 99\%$, compare Figure~\ref{fig:experiments/moons-size/interpolation}.}
        \label{fig:experiments/moons-size/pruning}
    \end{subfigure}
    
    \caption{Sample size limit on the Two Moons data set.}
    \label{fig:experiments/moons-size}
\end{figure}

\paragraph{Interpolation probability.}

The contour lines in the heatmap in Figure~\ref{fig:experiments/moons-size/interpolation} show which width of the first layer is required to achieve interpolation with a fixed probability for a certain number of samples. We can observe that there is an increase in the required width up to around $800\,000$ samples. After this threshold, however, a constant width of the first layer is enough to interpolate any number of samples.

\paragraph{Width of the second layer $\hat{\Phi}$.}

As in the other experiments we are interested in the part of the parameter space where the interpolation probability is almost one. Similar to the contour lines of the interpolation probability we observe that to obtain a fixed width of the second layer there is an increase in the required width of the first layer only up to a certain threshold (again, around $800\,000$ samples). After this threshold, a constant width of the first layer is enough to obtain a fixed width of the second layer.

\bigskip

Combining the above observations we note the following: there is a threshold in the number of samples such that for larger sample sizes there is a width of the first layer for which the network interpolates with probability close to one and the width of the second layer stays constant. Hence, as the width of the second layer is only lower for smaller sample sizes, a neural network of constant size (whose parameters can be computed via our algorithm) suffices to interpolate any number of samples.

\subsection{Binary Classification on MNIST}
\label{sec:experiments/binaryMNIST}

In the previous section, we ran a controlled experiment with a data generating distribution that was guaranteed to satisfy the assumptions of our main theorem and which could be used to draw an unlimited number of samples. It is natural to ask if the network size can also be observed to saturate in terms of the number of samples on real data. Examining binary classification subproblems of the MNIST data set~\citep{lecun1998mnist}, we find that the answer is `only sometimes'. We illustrate this in Figure~\ref{fig:mnist/one_vs_one}, which depicts the results for the `$1$ vs.\ $9$' and `$1$ vs.\ $8$' subproblems. For `$1$ vs.\ $9$', the second layer width clearly saturates as the number of samples grows. On the other hand, for `$1$ vs.\ $8$', although the curve seems to flatten a little, the second layer essentially grows linearly. For other binary subproblems, we observed that it was more common that the network size did not completely saturate. We emphasize that this does not contradict our claim that our approach yields a network of a size that is independent from the number of samples. Let us point to two possible explanations. First, MNIST may not contain enough samples to accurately represent the underlying distribution. Recall from Section~\ref{sec:experiments/sample-size-limit} that, even for the simple Two Moons data set, we needed around $800\, 000$ samples to demonstrate a clear saturation effect. Second, there may be an overlap in the class distributions which violates our separation assumption. This can happen quite easily for real data due to the presence of noise.

\begin{figure}
    \centering
    \begin{subfigure}{.5\textwidth}
        \centering
        \includegraphics{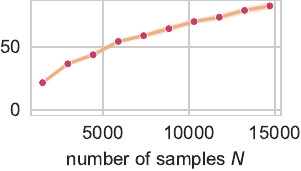}
        \subcaption{$1$ vs.\ $8$}
    \end{subfigure}%
    \begin{subfigure}{.5\textwidth}
        \centering
        \includegraphics{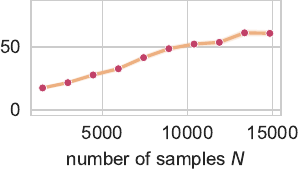}
        \subcaption{$1$ vs.\ $9$}
    \end{subfigure}
    \caption{\textbf{Width of the second layer $\hat{\Phi}$ on MNIST subproblems.} The width of the second layer $\hat{\Phi}$ (average over 25 runs) as a function of the number of samples $N$ for two binary classification subproblems of MNIST. We fixed $n = 5\, 000$ and $\lambda = 0.5$. Markers indicate an interpolation probability $\geq 99\%$, which is everywhere on all curves in this figure.}
    \label{fig:mnist/one_vs_one}
\end{figure}

\subsection{Multi-Class Classification on MNIST}
\label{sec:experiments/multiclass-classification}

Recall that our method is designed for binary problems. One-versus-many is a common strategy to extend binary classification methods to multi-class problems: for each class, train a binary classifier to distinguish between this class and all other classes. At inference time, query all classifiers and output the class label corresponding to the classifier with the highest confidence score.

We extend Algorithm~\ref{alg:numerical-experiments} to multi-class problems in a similar manner. However, as the first layer is obtained in an identical way for every execution of our method, we reuse it across all classes. One can use a simple union bound argument to prove high success probability for this case. Let $K \geq 2$ denote the total number of classes and $\set{X}_k$ the set of samples of class $k \in \range{K}$. Sample the first layer $\Phi$ at random as in Algorithm~\ref{alg:numerical-experiments}. Then, for each class $k \in \range{K}$ compute the second and third layer while using $\Phi$ as the first layer and $\set{X}^- = \set{X}_k$ and $\set{X}^+ = \bigcup_{\ell \not= k} \set{X}_\ell$ as input data. It is convenient to modify the third layer to map samples of $\set{X}^-$ to $1$ and samples of $\set{X}^+$ to $0$. Denote the concatenation of the second and third layers by $F_k$. Define the final classifier $F\colon (\set{X}_1 \cup \dots \cup \set{X}_K) \to \{0,1\}^K$ by
\begin{equation*}
    F(\vec{x}) = (F_1(\Phi(\vec{x})), \dots, F_K(\Phi(\vec{x})))
\end{equation*}
which outputs the class label as a one-hot encoding. We apply this method to MNIST.

\paragraph{Interpolation probability.}

In Figure~\ref{fig:experiments/mnist/interpolation} we again observe a clear phase transition in the interpolation probability. As in the case of Two Moons, this behaves as predicted by Theorem~\ref{thm:pruning/covering-bound}, as for $\lambda$ larger than the radius of the data, $n \gtrsim \lambda$ is enough to guarantee interpolation with any fixed probability. For $\lambda = 0$ the method not only interpolates but it does so with the narrowest first layer. That this works can be intuitively explained by the angular separation of MNIST. The minimal angle between two samples from MNIST is around $0.17$ (in contrast to about $2.44 \cdot 10^{-6}$ for Two Moons). Hence, it is possible to efficiently separate pairs of samples with hyperplanes through the origin.

\paragraph{Width of the second layer $\hat{\Phi}$.}

Again we are interested in the part of the parameter space where the interpolation probability is close to one. In Figure~\ref{fig:experiments/mnist/pruning} we observe that, while $\lambda = 0$ seems to be the optimal choice (for the interpolation probability), increasing $n$ may still lead to a reduction of the width of the second layer. Figure~\ref{fig:mnist/width-ext} reveals that the width does decrease well after interpolation is possible, and in fact, $\lambda \approx 0.5$ yields an even lower value. This might be due to the effect that can be seen in Figure~\ref{fig:experiments/moons/decision-boundary-grid}, where for $\lambda = 0$ the activation regions of the neurons of the second layer are `wedges' and become more circular for larger $\lambda$, which then might prove beneficial to the width of the second layer. Compared to the binary classification experiments in the previous sections, the width of the second layer is relatively large. For the most part, this is due to the larger number of classes: due to our one-versus-many approach, the width of the second layer scales as $\sum_{i=1}^K M_i^-$, where $K$ is the number of classes and $M_i^{-}$ is the mutual covering number for the one-versus-rest problem for class $i$. Additionally, MNIST may simply not admit a `small' mutual covering. Although the concept of mutual covering adapts to the relative positioning of the classes, it is not clear whether a covering with Euclidean balls yields the right complexity measure for image data. It would be an interesting future research direction to adapt our method to a different notion of covering that is more suitable for specific types of data, such as images.

\begin{figure}
    \centering
    \begin{subfigure}{\textwidth}
        \centering
        \includegraphics{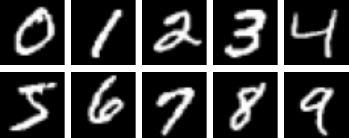}
        \caption{\textbf{MNIST.} A multi-class classification data set containing a total of $70.000$ grayscale images of handwritten digits. Each image has dimension $d = 28 \times 28 = 784$. We mapped the pixel values from $\{0, \dots, 255\}$ to $[0, 1]$ and normalized the radius to $R = 1$.}
        \label{fig:experiments/mnist/dataset}
    \end{subfigure}

    \begin{subfigure}{\textwidth}
        \centering
        \parbox{.95\textwidth}{
            \parbox{.45\textwidth}{%
                \includegraphics{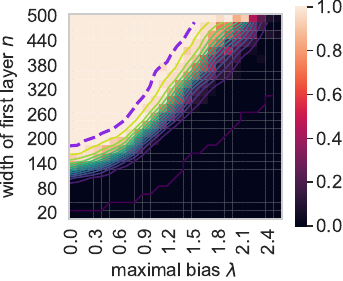}
            }
            \parbox{.45\textwidth}{%
                \includegraphics{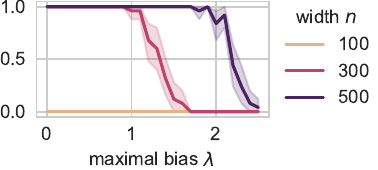}
                \includegraphics{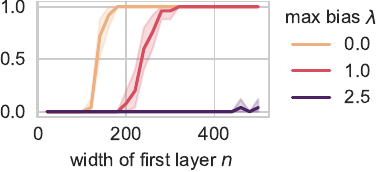}
            }
        }
        \caption{\textbf{Interpolation probability.} (Left) The interpolation probability (average over 25 runs) as a function of the width of the first layer $n$ and maximal bias $\lambda$. The 99\% contour line is at the dashed purple line. (Right) Horizontal (top) and vertical (bottom) slices of the heatmap with $95\%$ confidence intervals.}
        \label{fig:experiments/mnist/interpolation}
    \end{subfigure}

    \begin{subfigure}{\textwidth}
        \centering
        \parbox{.95\textwidth}{
            \parbox{.45\textwidth}{%
                \includegraphics{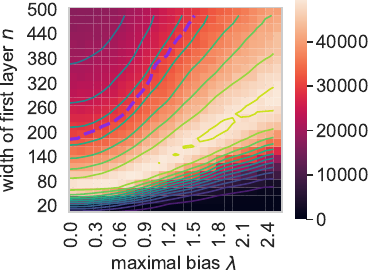}
            }
            \parbox{.45\textwidth}{%
                \includegraphics{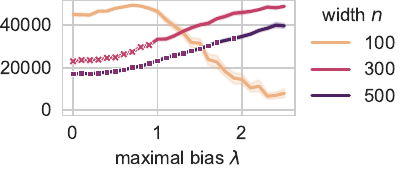}
                \includegraphics{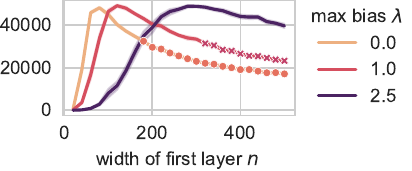}
            }
        }
        \caption{\textbf{Width of the second layer $\hat{\Phi}$.} (Left) The width of the second layer $\hat{\Phi}$ (average over 25 runs) as a function of the width of the first layer $n$ and maximal bias $\lambda$. The 99\% interpolation probability contour line is at the dashed purple line. (Right) Horizontal (top) and vertical (bottom) slices of the heatmap with $95\%$ confidence intervals. Markers indicate an interpolation probability $\geq 99\%$, compare Figure~\ref{fig:experiments/mnist/interpolation}.}
        \label{fig:experiments/mnist/pruning}
    \end{subfigure}
    
    \caption{Multi-class classification on the MNIST data set.}
    \label{fig:experiments/mnist}
\end{figure}

\begin{figure}
    \centering
    \begin{subfigure}{.5\textwidth}
        \includegraphics{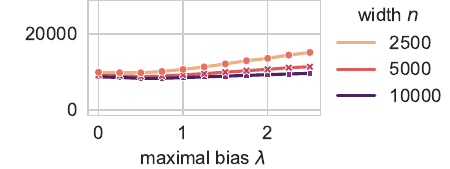}
    \end{subfigure}%
    \begin{subfigure}{.5\textwidth}
        \includegraphics{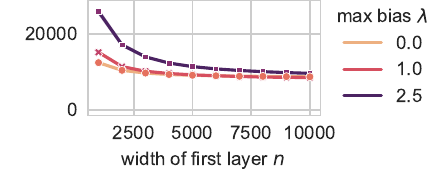}
    \end{subfigure}
    \caption{\textbf{Width of the second layer $\hat{\Phi}$ on MNIST.} Horizontal (left) and vertical (right) slices of the heatmap in Figure~\ref{fig:experiments/mnist/pruning} for an extended range of the width $n$ of the first layer. Markers again indicate an interpolation probability $\geq 99\%$, which is everywhere on all curves in this figure.}
    \label{fig:mnist/width-ext}
\end{figure}

\subsection{Multi-Class Classification on CIFAR-10}
\label{sec:experiments/cifar10}

Next, we apply the extension for multi-class problems from the previous section to CIFAR-10. Due to the color channels and a slightly higher resolution, the dimension is larger than that of MNIST. Additionally, photos of real objects provide more variety than handwritten digits.

\paragraph{Interpolation probability.}

In Figure~\ref{fig:experiments/cifar10/interpolation} we see a clear phase transition in the interpolation probability. As in the other experiments, this behaves as predicted by the Theorem~\ref{thm:pruning/covering-bound}: for $\lambda$ larger than the data radius, $n \gtrsim \lambda$ yields interpolation for any fixed probability. As with MNIST in the previous section, $\lambda = 0$ is the best choice.

\paragraph{Width of the second layer $\hat{\Phi}$.}

In Figure~\ref{fig:experiments/cifar10/pruning} we observe that the width of the second layer seems almost constant in the part of the parameter space where the interpolation probability is close to one. Considering even larger values of $n$ in Figure~\ref{fig:cifar10/width-ext}, the width of the second layer decreases well beyond the interpolation threshold and the optimal choice of the maximal bias seems to be around $\lambda = 0.25$. Again, relative to the number of samples, the width of the second layer is very large. As in the case of MNIST in the previous section, one might conjecture that the data is either ill-conditioned in terms of the mutual covering or violates one of our assumptions. We will come back to this in Section~\ref{sec:conclusion}.

\begin{figure}
    \centering
    \begin{subfigure}{\textwidth}
        \centering
        \includegraphics{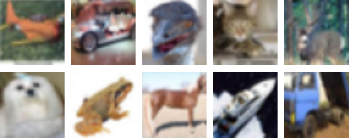}
        \caption{\textbf{CIFAR-10.} The CIFAR-10 data set consists of $60\,000$ color images in $10$ classes of different objects, with $6\,000$ images per class. Each image has dimension $d = 32 \times 32 \times 3 = 3072$. The pixel values reside in $[0, 1]$ and we normalized the radius to $R = 1$.}
        \label{fig:experiments/cifar10/dataset}
    \end{subfigure}

    \begin{subfigure}{\textwidth}
        \centering
        \parbox{.95\textwidth}{
            \parbox{.45\textwidth}{%
                \includegraphics{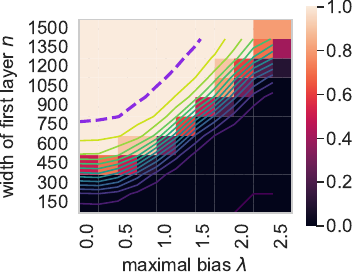}
            }
            \parbox{.45\textwidth}{%
                \includegraphics{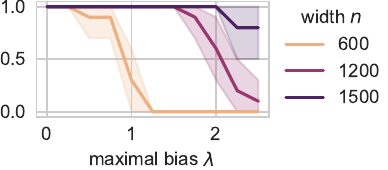}
                \includegraphics{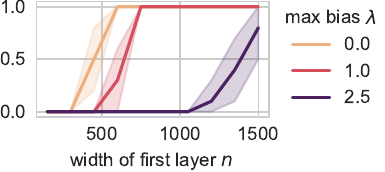}
            }
        }
        \caption{\textbf{Interpolation probability.} (Left) The interpolation probability (average over $10$ runs) as a function of the width of the first layer $n$ and maximal bias $\lambda$. The 99\% contour line is at the dashed purple line. (Right) Horizontal (top) and vertical (bottom) slices of the heatmap with $95\%$ confidence intervals.}
        \label{fig:experiments/cifar10/interpolation}
    \end{subfigure}

    \begin{subfigure}{\textwidth}
        \centering
        \parbox{.95\textwidth}{
            \parbox{.45\textwidth}{%
                \includegraphics{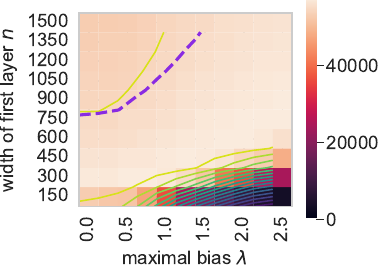}
            }
            \parbox{.45\textwidth}{%
                \includegraphics{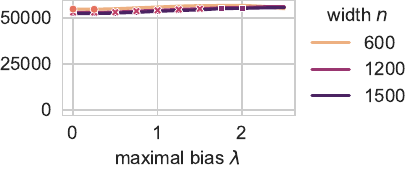}
                \includegraphics{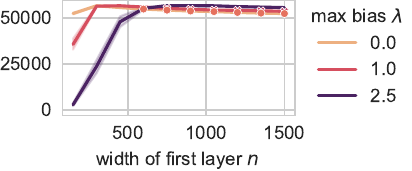}
            }
        }
        \caption{\textbf{Width of the second layer $\hat{\Phi}$.} (Left) The width of the second layer $\hat{\Phi}$ (average over 10 runs) as a function of the width of the first layer $n$ and maximal bias $\lambda$. The 99\% interpolation probability contour line is at the dashed purple line. (Right) Horizontal (top) and vertical (bottom) slices of the heatmap with $95\%$ confidence intervals. Markers indicate an interpolation probability $\geq 99\%$, compare Figure~\ref{fig:experiments/mnist/interpolation}.}
        \label{fig:experiments/cifar10/pruning}
    \end{subfigure}
    
    \caption{Multi-class classification on the CIFAR-10 data set.}
    \label{fig:experiments/cifar10}
\end{figure}

\begin{figure}
    \centering
    \begin{subfigure}{.5\textwidth}
        \includegraphics{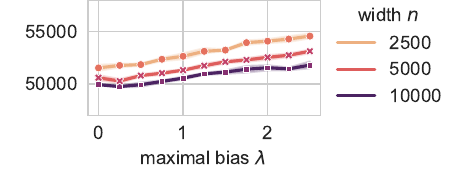}
    \end{subfigure}%
    \begin{subfigure}{.5\textwidth}
        \includegraphics{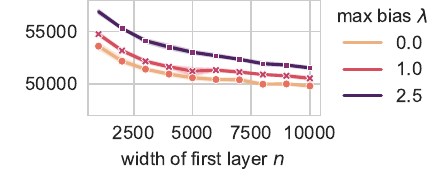}
    \end{subfigure}
    \caption{\textbf{Width of the second layer $\hat{\Phi}$ on CIFAR-10.} Horizontal (left) and vertical (right) slices of the heatmap in Figure~\ref{fig:experiments/cifar10/pruning} for an extended range of the width $n$ of the first layer. Markers again indicate an interpolation probability $\geq 99\%$, which is everywhere on all curves in this figure.}
    \label{fig:cifar10/width-ext}
\end{figure}

\subsection{A Worst-Case Example}
\label{sec:experiments/worst-case}

We conclude with a constructed example that demonstrates that our algorithm can in certain cases fail to produce a small interpolating net. Figure~\ref{fig:lines-worst-case} shows samples drawn from two parallel lines, where the distances of samples between classes are smaller than the distances of samples within each class. This forces the components of the mutual covering (and the activation regions of the neurons in the second layer) to be so small that they only cover a single point. Hence, the width of the second layer scales as the number of samples, which is the worst case. This example shows that, although our algorithm is guaranteed to produce small interpolating neural networks on data with a small mutual covering number, it may not take advantage of alternative benign structures (linear separability in this constructed example). 

\begin{figure}
    \centering
    \includegraphics{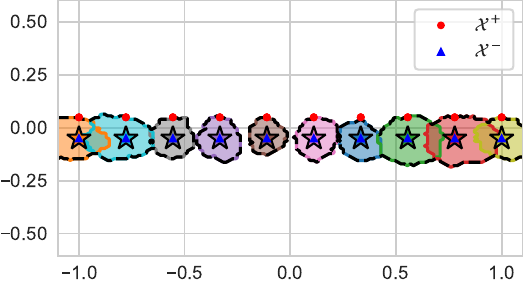}
    \caption{\textbf{Parallel lines.} Points are sampled from two parallel lines such that the distance of samples between classes is smaller than the distance of samples within each class. In this case, each neuron in the second layer activates only for its associated point, and hence the second layer has maximal width. Note, as all points of $\set{X}^-$ are accepted into the second layer, they are all marked with stars. Here, we used $n = 2\,000$ and $\lambda = 1.0$.}
    \label{fig:lines-worst-case}
\end{figure}

\section{Conclusion}
\label{sec:conclusion}

In this paper, we presented an instance-specific viewpoint on the memorization problem for neural networks. We quantified the sufficient network size that guarantees interpolation of given data with two classes in terms of a mutual covering that takes both the geometric complexities and the mutual arrangement of the classes into account. Under our assumptions, the network size depends only on the mutual covering and does not depend on the number of samples in the data set. In this way, our result moves beyond worst-case memorization capacity bounds, which cannot be independent of the number of samples. We gave a constructive proof by presenting a randomized algorithm that is guaranteed to produce an interpolating network for given input data with high probability. We illustrated our theoretical guarantees, in particular the independence of the number of samples, by testing our randomized interpolation algorithm in controlled numerical experiments. In addition, we tested our algorithm on image data and found that it produced relatively large interpolating networks in many cases. In future work, we aim to improve our algorithm for real data by making it robust to noise and by making it tailored to low-complexity structures present in real data such as images.


\newpage

\acks{S.D.\ and M.G.\ acknowledge support by the DFG Priority Programme DFG-SPP 1798 Grant DI 2120/1-1. The authors thank the anonymous reviewers for their comments and suggestions that lead to improvements in our work.}


\vskip 0.2in
\bibliography{ref}

\end{document}